\documentclass{article}

\usepackage{microtype}
\usepackage{graphicx}
\usepackage{booktabs} 

\usepackage{hyperref}

\usepackage[utf8]{inputenc} \usepackage[T1]{fontenc}    \usepackage{hyperref}       \usepackage{url}            \usepackage{booktabs}       \usepackage{amsfonts}       \usepackage{nicefrac}       \usepackage{microtype}      \usepackage{xspace}
\usepackage{xcolor}
\usepackage{graphicx}
\usepackage{amsmath}
\usepackage{amsthm}
\usepackage{algorithm}
\usepackage{algorithmicx} \usepackage{algpseudocode} \usepackage[font=small]{caption}
\usepackage{subcaption}
\usepackage{mathtools}
\usepackage{enumerate}
\usepackage{enumitem}
\usepackage{wrapfig}
\usepackage{booktabs}
\usepackage{thmtools}
\usepackage[export]{adjustbox}
\declaretheorem[name=Proposition]{prop}

\newcommand{\ours}{\textsc{Dream}\xspace}
\newcommand{\import}{\textsc{Import}\xspace}
\newcommand{\varibad}{\textsc{VariBAD}\xspace}
\newcommand{\pearl}{\textsc{Pearl}\xspace}
\newcommand{\pearlub}{\textsc{Pearl-UB}\xspace}
\newcommand{\rl}{RL$^2$\xspace}
\newcommand{\norm}[1]{\left\lVert#1\right\rVert}

\newcommand{\mdpindex}{\mu}
\newcommand{\problems}{\sM}
\newcommand{\rewardpenalty}{c}
\newcommand{\paramexp}{\phi}
\newcommand{\paramin}{\theta}
\newcommand{\paramF}{\psi}
\newcommand{\paramdecoder}{\omega}

\newcommand{\piexp}{\pi^\text{exp}}
\newcommand{\qexp}{\hat{Q}^\text{exp}}
\newcommand{\piin}{\pi^\text{task}}
\newcommand{\qin}{\hat{Q}^\text{task}}
\newcommand{\vin}{\hat{V}^\text{task}}
\newcommand{\actions}{\mathbf{a}}
\newcommand{\rexp}{r^\text{exp}}
\newcommand{\tauexp}{\tau^\text{exp}}

\newcommand{\bartauexp}{\tau^\text{exp}_\text{good}}
\newcommand{\underbartauexp}{\tau^{\text{exp}}_\text{bad}}

\newcommand{\iclrupdate}[1]{#1}
\newcommand{\cameraupdate}[1]{#1}

\newtheorem*{proposition*}{Proposition}{}

\newcommand\sA{\ensuremath{\mathcal{A}}}
\newcommand\sB{\ensuremath{\mathcal{B}}}

\newcommand\sJ{\ensuremath{\mathcal{J}}}

\newcommand\sL{\ensuremath{\mathcal{L}}}
\newcommand\sM{\ensuremath{\mathcal{M}}}
\newcommand\sN{\ensuremath{\mathcal{N}}}
\newcommand\sO{\ensuremath{\mathcal{O}}}

\newcommand\sR{\ensuremath{\mathcal{R}}}
\newcommand\sS{\ensuremath{\mathcal{S}}}

         \newcommand{\bzero}{\mathbf{0}} \newcommand\refeqn[1]{(\ref{eqn:#1})}

\newcommand\refsec[1]{Section~\ref{sec:#1}}

\newcommand\reffig[1]{Figure~\ref{fig:#1}}

\newcommand\reftab[1]{Table~\ref{tab:#1}}
\newcommand\refapp[1]{Appendix~\ref{sec:#1}}

\newcommand\refalg[1]{Algorithm~\ref{alg:#1}}

\ifthenelse{\isundefined{\definition}}{\newtheorem{definition}{Definition}}{}
\ifthenelse{\isundefined{\assumption}}{\newtheorem{assumption}{Assumption}}{}
\ifthenelse{\isundefined{\hypothesis}}{\newtheorem{hypothesis}{Hypothesis}}{}
\ifthenelse{\isundefined{\proposition}}{\newtheorem{proposition}{Proposition}}{}
\ifthenelse{\isundefined{\theorem}}{\newtheorem{theorem}{Theorem}}{}
\ifthenelse{\isundefined{\lemma}}{\newtheorem{lemma}{Lemma}}{}
\ifthenelse{\isundefined{\corollary}}{\newtheorem{corollary}{Corollary}}{}
\ifthenelse{\isundefined{\alg}}{\newtheorem{alg}{Algorithm}}{}
\ifthenelse{\isundefined{\example}}{\newtheorem{example}{Example}}{}
      \newcommand{\E}{\ensuremath{\mathbb{E}}}  %

\usepackage[accepted]{icml2021}

\icmltitlerunning{Decoupling Exploration and Exploitation for Meta-Reinforcement Learning}

\begin{document}

\twocolumn[
\icmltitle{Decoupling Exploration and Exploitation for Meta-Reinforcement Learning without Sacrifices}

\icmlsetsymbol{equal}{*}

\begin{icmlauthorlist}
\icmlauthor{Evan Zheran Liu}{stan}
\icmlauthor{Aditi Raghunathan}{stan}
\icmlauthor{Percy Liang}{stan}
\icmlauthor{Chelsea Finn}{stan}
\end{icmlauthorlist}

\icmlaffiliation{stan}{Department of Computer Science, Stanford University}

\icmlcorrespondingauthor{Evan Zheran Liu}{evanliu@cs.stanford.edu}

\icmlkeywords{Machine Learning, ICML}

\vskip 0.3in
]

\printAffiliationsAndNotice{}  

\newcommand{\fix}{\marginpar{FIX}}
\newcommand{\new}{\marginpar{NEW}}

\everypar{\looseness=-1}
\linepenalty=1000

\begin{abstract}
The goal of meta-reinforcement learning (meta-RL) is to build agents that can quickly learn new tasks by leveraging prior experience on related tasks.
    Learning a new task often requires both exploring to gather task-relevant information and exploiting this information to solve the task.
In principle, optimal exploration and exploitation can be learned end-to-end by simply maximizing task performance.
However, such meta-RL approaches struggle with local optima due to a chicken-and-egg problem:
learning to explore requires good exploitation to gauge the exploration’s utility, but learning to exploit requires information gathered via exploration.
Optimizing separate objectives for exploration and exploitation can avoid this problem, but prior meta-RL exploration objectives yield suboptimal policies that gather information irrelevant to the task.
    We alleviate both concerns by constructing an exploitation objective that automatically identifies task-relevant information and an exploration objective to recover only this information.
This avoids local optima in end-to-end training, without sacrificing optimal exploration.
    Empirically, \ours substantially outperforms existing approaches on complex meta-RL problems, such as sparse-reward 3D visual navigation.
    Videos of \ours: \url{https://ezliu.github.io/dream/}

\end{abstract}
 \section{Introduction}

A general-purpose agent should be able to perform multiple related tasks across multiple related environments.
Our goal is to develop agents that can perform a variety of tasks in novel environments, based on previous experience and only a small amount of experience in the new environment.
For example, we may want a robot to cook a meal (a new task) in a new kitchen (the environment) after it has learned to cook other meals in other kitchens.
To adapt to a new kitchen, the robot must both explore to find the ingredients, and use this information to cook.
Existing meta-reinforcement learning (meta-RL)
methods can adapt to new tasks and environments, but, as we identify in this work, struggle when adaptation requires complex exploration strategies.

In the meta-RL setting, the agent is presented with a set of meta-training problems, each in an environment (e.g., a kitchen) with some task (e.g., make pizza);
at meta-test time, the agent is given a new, but related environment and task.
It is allowed to gather information in a few initial exploration episodes, and its goal is to then maximize returns on all subsequent exploitation episodes, using this information.
A common meta-RL approach is to learn to explore and exploit \emph{end-to-end} by training a policy and updating exploration behavior based on how well the policy later exploits using the information discovered from exploration~\citep{duan2016rl, wang2016learning, stadie2018importance, zintgraf2019varibad, humplik2019meta}.
With enough model capacity, such approaches can express optimal exploration and exploitation,
but they create a chicken-and-egg problem that leads to bad local optima and poor sample efficiency:
Learning to explore requires good exploitation to gauge the exploration's utility, but learning to exploit requires information gathered via exploration. Therefore, with only final performance as signal, one cannot be learned without already having learned the other.
For example, a robot chef is only incentivized to explore and find the ingredients if it already knows how to cook with those ingredients,
but the robot can only learn to cook if it can already find the ingredients by exploration.

To avoid this chicken-and-egg problem, we propose to optimize separate objectives for exploration and exploitation by leveraging the \emph{problem ID}---an easy-to-provide unique one-hot for each training meta-training task and environment.
Such a problem ID can be realistically available in real-world meta-RL tasks:
e.g., in a robot chef factory, each training kitchen (problem) can be easily assigned a unique ID, and in a personalized recommendation system, each user (problem) is typically identified by a unique username.
Some prior works~\citep{humplik2019meta, kamienny2020learning} also use these problem IDs, but not in a way that avoids the chicken-and-egg problem.
Others~\citep{rakelly2019efficient, zhou2019environment, gupta2018meta, gurumurthy2019mame, zhang2020learn} also optimize separate objectives, but their exploration objectives learn suboptimal policies that needlessly gather task-irrelevant information.

Instead, we propose an exploitation objective that automatically identifies task-relevant information, and an exploration objective to recover only this information.
We learn an exploitation policy without the need for exploration, by conditioning on a learned representation of the problem ID, which provides task-relevant information
(e.g., by memorizing the locations of the ingredients for each ID / kitchen).
We apply an information bottleneck to this representation to encourage discarding of any information not required by the exploitation policy (i.e., task-irrelevant information).
Then, we learn an exploration policy to only discover task-relevant information by training it to produce trajectories containing the same information as the learned ID representation (\refsec{approach}).
Crucially, unlike prior work, we prove that our separate objectives are \emph{consistent}:
optimizing them yields optimal exploration and exploitation, assuming expressive-enough policy classes and enough meta-training data (\refsec{consistency}).

Overall, this work's main contribution is a consistent decoupled meta-RL algorithm, called \ours:
\textbf{D}ecoupling explo\textbf{R}ation and \textbf{E}xploit\textbf{A}tion in \textbf{M}eta-RL, which overcomes the chicken-and-egg problem (\refsec{dream}).
Theoretically, in a simple tabular example, we show that addressing the coupling problem with \ours provably improves sample complexity over existing end-to-end approaches by a factor exponential in the horizon (\refsec{analysis}).
Empirically, we stress test \ours's ability to learn sophisticated exploration strategies on 3 challenging, didactic benchmarks and a sparse-reward 3D visual navigation benchmark.
On these,
\ours learns to optimally explore and exploit, achieving 90\% higher returns than existing state-of-the-art approaches (\pearl, E-\rl, \import, \varibad), which struggle to learn an effective exploration strategy (\refsec{experiments}).

 \section{Related Work}\label{sec:related_work}

We draw on a rich literature on learning to adapt to related tasks \citep{schmidhuber1987evolutionary, thrun2012learning, naik1992meta, bengio1991learning, bengio1992optimization, hochreiter2001learning, andrychowicz2016learning, santoro2016one}.
Many meta-RL works focus on adapting efficiently to a new task from few samples without optimizing the sample collection process,
via
updating the policy parameters \citep{finn2017modelagnostic, agarwal2019learning, yang2019norml, houthooft2018evolved, mendonca2019guided},
learning a model \citep{nagabandi2018learning, saemundsson2018meta, hiraoka2020meta},
multi-task learning \citep{fakoor2019meta},
or leveraging demonstrations \citep{zhou2019watch}.
In contrast, we focus on problems where targeted exploration is critical for few-shot adaptation.

Approaches that specifically explore to obtain the most informative samples fall into two main categories:
\emph{end-to-end} and \emph{decoupled} approaches.
End-to-end approaches optimize exploration and exploitation end-to-end by updating exploration behavior from returns achieved by exploitation~\citep{duan2016rl, wang2016learning, mishra2017simple, rothfuss2018promp, stadie2018importance, zintgraf2019varibad, humplik2019meta, kamienny2020learning, dorfman2020offline}.
These approaches can represent the optimal policy~\citep{kaelbling1998planning}, but
they struggle to escape local optima due to a chicken-and-egg problem between learning to explore and learning to exploit (\refsec{cyclic}).
Several of these approaches~\citep{humplik2019meta, kamienny2020learning} also leverage the problem ID during meta-training, but they still learn end-to-end, so the chicken-and-egg problem remains.

Decoupled approaches instead optimize separate exploration and exploitation objectives, via, e.g., Thompson-sampling (TS) \citep{thompson1933likelihood, rakelly2019efficient}, obtaining exploration trajectories predictive of dynamics or rewards \citep{zhou2019environment, gurumurthy2019mame,zhang2020learn}, or exploration noise \citep{gupta2018meta}.
While these works do not identify the chicken-and-egg problem, decoupled approaches coincidentally avoid it.
However, existing decoupled approaches, including those~\citep{rakelly2019efficient, zhang2020learn} that leverage the problem ID, do not learn optimal exploration:
TS \citep{rakelly2019efficient} explores by guessing the task and executing a policy for that task,
and hence cannot represent exploration behaviors that are different from exploitation~\citep{russo2017tutorial}.
Predicting the dynamics \citep{zhou2019environment,gurumurthy2019mame,zhang2020learn} is inefficient 
when only a small subset of the dynamics are relevant to solving the task.
In contrast, we propose a separate mutual information objective for exploration, which both avoids the chicken-and-egg problem and yields optimal exploration when optimized (\refsec{analysis}).
Past work \citep{gregor2016variational, houthooft2016vime, eysenbach2018diversity, warde2018unsupervised} also optimize mutual information objectives, but not for meta-RL.

Beyond meta-RL, learning a policy in the general RL setting (i.e., learning from scratch) also requires targeted exploration to gather informative samples.
In contrast to exploration algorithms for general RL \citep{bellemare2016unifying, pathak2017curiosity, burda2018exploration, leibfried2019unified}, which must visit many novel states to find regions with high reward, exploration in meta-RL can be substantially more targeted by leveraging prior experience from related problems during meta-training.
As a result, \ours can learn new tasks in just \emph{two} episodes (\refsec{experiments}), while learning from scratch can require millions of episodes to learn a new task. \begin{figure*}\center
\includegraphics[width=\textwidth]{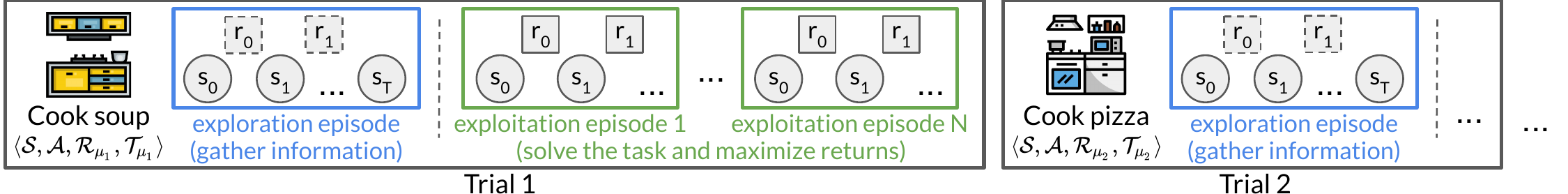}
\vspace{-7mm}
\caption{
\small
Meta-RL setting:
Given a new environment and task, the agent is allowed to first explore and gather information (exploration episode), and then must use this information to solve the task (in subsequent exploitation episodes).
}\label{fig:setting}
\end{figure*}

\section{Preliminaries}\label{sec:prelims}

\textbf{Meta-reinforcement learning.}
The meta-RL setting considers a family of Markov decision processes (MDPs)
$\langle \mathcal{S}, \mathcal{A}, \mathcal{R}_\mdpindex, T_\mdpindex \rangle$ 
with states $\mathcal{S}$, actions $\mathcal{A}$, rewards $\mathcal{R}_\mdpindex$, and dynamics $T_\mdpindex$, indexed by a one-hot \emph{problem ID} $\mdpindex \in \problems$, drawn from a distribution $p(\mu)$.
Colloquially, we refer to the dynamics as the \emph{environment}, the rewards as the \emph{task}, and the entire MDP as the \emph{problem}. 
Borrowing terminology from \citet{duan2016rl}, meta-training and meta-testing both consist of repeatedly running \emph{trials}.
Each trial consists of sampling a problem ID $\mdpindex \sim p(\mdpindex)$ and running $N + 1$ episodes on the corresponding problem.
Following prior evaluation settings~\citep{finn2017modelagnostic, rakelly2019efficient, rothfuss2018promp, fakoor2019meta}, we designate the first episode in a trial as an \emph{exploration} episode consisting of $T$ steps for gathering information, and define the goal as maximizing the returns in the subsequent $N$ \emph{exploitation} episodes (\reffig{setting}).
Following \citet{rakelly2019efficient, humplik2019meta, kamienny2020learning}, the easy-to-provide problem ID is available for meta-training, but not meta-testing trials.

We formally express the goal in terms of an exploration policy $\piexp$ used in the exploration episode and an exploitation policy $\piin$ used in exploitation episodes, but these policies may be the same or share parameters.
Rolling out $\piexp$ in the exploration episode produces an exploration trajectory $\tauexp = (s_0, a_0, r_0, \ldots, s_T)$, which contains information discovered via exploration.
The exploitation policy $\piin$ may then condition on $\tauexp$ and optionally, its history across all exploitation episodes in a trial, to maximize exploitation episode returns.
The goal is therefore to maximize:
\begin{equation}
\label{eqn:standard_objective}
    \mathcal{J}(\piexp, \piin) = \E_{\mdpindex \sim p(\mdpindex), \tauexp \sim \piexp}\left[
        V^{\text{task}}(\tauexp; \mdpindex)
    \right],
\end{equation}
where $V^{\text{task}}(\tauexp; \mdpindex)$ is the expected returns of $\piin$ conditioned on $\tauexp$, summed over the $N$ exploitation episodes in a trial with problem ID $\mdpindex$.

\textbf{End-to-end meta-RL.}
A common meta-RL approach \citep{wang2016learning, duan2016rl, rothfuss2018promp, zintgraf2019varibad, kamienny2020learning, humplik2019meta} is to learn to explore and exploit \emph{end-to-end} by directly optimizing $\sJ$ in \refeqn{standard_objective}, updating both from rewards achieved during exploitation.
These approaches typically learn a single recurrent policy $\pi(a_t \mid s_t, \tau_{:t})$ for both exploration and exploitation (i.e., $\piin = \piexp = \pi$), which takes action $a_t$ given state $s_t$ and history of experiences spanning all episodes in a trial $\tau_{:t} = (s_0, a_0, r_0, \ldots, s_{t - 1}, a_{t - 1}, r_{t - 1})$.
Intuitively, this policy is learned by rolling out a trial,
producing an exploration trajectory $\tauexp$ and, conditioned on $\tauexp$ and the exploitation experiences so far, yielding some exploitation episode returns.
Then, credit is assigned to both exploration (producing $\tauexp$) and exploitation by backpropagating the exploitation returns through the recurrent policy.
Directly optimizing the objective $\sJ$ this way can learn
optimal exploration and exploitation strategies,
but optimization is challenging, which we detail in the next section.
 \section{Decoupling Exploration and Exploitation}\label{sec:approach}

\cameraupdate{In this section, we first illustrate how end-to-end optimization approaches face a chicken-and-egg problem between learning exploration and exploitation, leading to local optima and poor sample complexity (\refsec{cyclic}).
Next, in \refsec{dream}, we propose \ours to sidestep this chicken-and-egg problem by optimizing separate objectives for exploration and exploitation.
Finally, we describe a practical implementation of \ours in \refsec{practical_dream}.
Prior decoupled approaches also optimize separate exploration and exploitation objectives.
However, crucially, as we show in the next section, the optimum of \ours's objectives maximizes returns, while the optimum of prior objectives does not.
}

\subsection{The Problem with Coupling Exploration and Exploitation}
\label{sec:cyclic}

We begin by showing that end-to-end optimization struggles with local optima due to a chicken-and-egg problem, illustrated in \reffig{approach}.
Learning $\piexp$ relies on gradients passed through $\piin$.
If $\piin$ cannot effectively solve the task, then these gradients will be uninformative.
However, to learn to efficiently solve the task, $\piin$ needs good exploration data (trajectories $\tauexp$) from a good exploration policy $\piexp$.
This results in bad local optima as follows:
if our current (suboptimal) $\piin$ obtains low rewards with a good informative trajectory $\bartauexp$, the low reward would cause $\piexp$ to learn to \emph{not} generate $\bartauexp$.
This causes $\piexp$ to instead generate trajectories $\underbartauexp$ that lack information required to obtain high reward, further preventing the exploitation policy $\piin$ from learning.
Typically, early in training, both $\piexp$ and $\piin$ are suboptimal and hence will likely reach this local optimum.

More succinctly, estimates of the expected exploitation returns $V^\text{task}(\tauexp; \mdpindex)$ in \refeqn{standard_objective} (e.g., from value-function approximation) form the learning signal for exploration.
Escaping the local optima requires accurately estimating $V^\text{task}$, which requires many episodes, leading to sample inefficiency.
In \refsec{sample_complexity}, we illustrate this in a simple example.

\begin{figure}\center
\includegraphics[width=0.8\columnwidth]{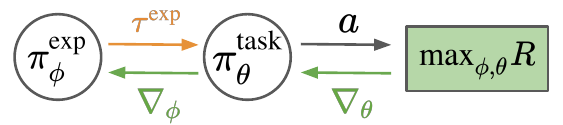}
\vspace{-2mm}
\caption{
\small
Coupling between the exploration policy $\piexp$ and exploitation policy $\piin$.
These policies are illustrated separately for clarity, but may be a single policy.
Since the two policies depend on each other (for gradient signal and the $\tauexp$ distribution), it is challenging to learn one when the other policy has not learned.
}\label{fig:approach}
\vspace{-3mm}
\end{figure}

\subsection{\ours: Decoupling Exploration and Exploitation in Meta-Reinforcement Learning}\label{sec:dream}
While we can sidestep the local optima of end-to-end training by optimizing separate objectives for exploration and exploitation, the challenge is to construct objectives that yield the same optimal solution as the end-to-end approach.
We now discuss how we can use the problem IDs during meta-training to do so.
Intuitively, a good exploration objective should encourage discovering task-relevant distinguishing attributes of the problem (e.g., ingredient locations), and ignoring task-irrelevant attributes (e.g., wall color).
To create this objective, the key idea behind \ours is to \emph{learn} to extract only the task-relevant information from the problem ID, which encodes all information about the problem.
Then, \ours's exploration objective seeks to recover only this task-relevant information.

Concretely, \ours extracts only the task-relevant information from the problem ID $\mdpindex$ via a stochastic encoder $F_\paramF(z \mid \mdpindex)$.
To learn this encoder, we train an exploitation policy $\piin(a \mid s, z)$ to maximize rewards, conditioned on samples $z \sim F_\paramF(z \mid \mdpindex)$, while simultaneously applying an information bottleneck to $z$ to discard information not needed by $\piin$ (i.e., task-irrelevant information).
Then, \ours learns an exploration policy $\piexp$ to produce trajectories with high mutual information with $z$.
In this approach, the exploitation policy $\piin$ no longer relies on effective exploration from $\piexp$ to learn, and once $F_\paramF(z \mid \mdpindex)$ is learned, the exploration policy also learns independently from $\piin$, decoupling the two optimization processes.
During meta-testing, $\mdpindex$ is either unavailable or uninformative because it is simply a novel one-hot ID.
However, the two policies can be easily combined,
since the trajectories generated by $\piexp$ are optimized to contain the same information as the encodings $z \sim F_\paramF(z \mid \mdpindex)$ that the exploitation policy $\piin$ trained on.
Next we describe each of these components in detail.

\textbf{Learning the problem ID encodings and exploitation policy.}
We first learn a stochastic encoder $F_\paramF(z \mid \mdpindex)$ parametrized by $\paramF$ 
and exploitation policy $\piin_\paramin(a \mid s, z)$
parametrized by $\paramin$, which conditions on $z$, by solving the following constrained optimization problem:
\begin{align}
\label{eqn:bottleneck-problem}
& \text{minimize}
& & I(z; \mdpindex) \\
& \text{subject to}
& &  \E_{z \sim F_\paramF(z \mid \mdpindex)}\left[
    V^{\piin_\paramin}(z; \mdpindex)\right] = V^{*}(\mdpindex) \textrm{ for all }\mdpindex, \nonumber
\end{align}
where $V^{\piin_\paramin}(z; \mdpindex)$ is the expected returns of $\piin_\paramin$ on problem $\mdpindex$, given encoding $z$, and $V^{*}(\mdpindex)$ is the maximum expected returns achievable by any policy on problem $\mdpindex$.
Intuitively, optimizing this problem discards any (task-irrelevant) information from $z$ (the objective) that does not help maximize returns (the constraint), and importantly, is independent of exploration.

In practice, we solve this problem (without knowing $V^*(\mdpindex)$), by maximizing the Lagrangian, with dual variable $\lambda^{-1}$:
\begin{align}
\label{eq:obj-encoding}
\underset{\paramF, \paramin}{\text{maximize}}~ \underbrace{\E_{\mdpindex \sim p(\mdpindex), z \sim F_\paramF(z \mid \mdpindex)}\left[
        V^{\piin_\paramin}(z; \mdpindex)
    \right]}_{\text{Returns}} -
\lambda\!\!\!\!\!\!\!\!\!\!\underbrace{\vphantom{\E_{\mdpindex \sim p(\mdpindex), z \sim F_\paramF(z \mid \mdpindex)}\left[V^{\piin_\paramin}\right]}
    I(z; \mdpindex).}_{\text{Information bottleneck}}
\end{align}
We 
maximize the returns via standard RL and
minimize the mutual information $I(z; \mdpindex)$ by minimizing a variational upper bound on it \citep{alemi2016deep},
$\E_\mdpindex\left[ \text{D}_{\text{KL}} (F_\paramF(z \mid \mdpindex) || j(z)) \right]$, where $j$ is any prior
and $z$ is distributed as $p_\paramF(z) = \int_\mdpindex F_\paramF(z \mid \mdpindex) p(\mdpindex) d\mdpindex$.
\cameraupdate{Note that the returns are optimized with respect to both the exploitation policy $\piin_\paramin$ and the encoder $F_\paramF$, while the information bottleneck only depends on and is only optimized with respect to $F_\paramF$.
}

\textbf{Learning an exploration policy given problem ID encodings.}
Once we've obtained an encoder $F_\paramF(z \mid \mdpindex)$ to extract only the necessary task-relevant information required to optimally solve each task,
we can optimize the exploration policy $\piexp$ to produce trajectories that contain this same information by maximizing their mutual information $I(\tauexp; z)$.
We slightly abuse notation to use $\piexp$ to denote the probability distribution over the trajectories $\tauexp$.
Then, the mutual information $I(\tauexp; z)$ can be efficiently maximized by maximizing a variational lower bound \citep{barber2003algorithm} as follows:
\begin{align}
\label{eqn:obj-exp}
    &I(\tauexp; z)
= H(z) - H(z \mid \tauexp) \\
  &\; \geq H(z) + \E_{\mdpindex, z \sim F_\paramF, \tauexp \sim \piexp} \left[\log q_\paramdecoder(z \mid \tauexp) \right]\nonumber \\ 
  &\; = H(z) + \E_{\mdpindex, z\sim F_\paramF}[\log q_\paramdecoder(z \mid s_0)]\;+ \nonumber\\
  &\quad\quad \E_{\mdpindex, z\sim F_\paramF, \tauexp \sim \piexp}\left[
\sum \limits_{t=0}^{T - 1} \log\frac{q_\paramdecoder(z \mid \tauexp_{:t + 1})}{q_\paramdecoder(z \mid \tauexp_{:t})}
  \right], \nonumber
\end{align}
where $q_\paramdecoder$ is any distribution parametrized by $\paramdecoder$,
and $\tau^\text{exp}_{:t}$ denotes the trajectory up to the $t$-th state $(s_0, a_0, r_0, \ldots, s_t)$.
The last line comes from expanding a telescoping series.
We maximize the above expression over $\tauexp$ and over $\paramdecoder$ to learn $q_\paramdecoder$ that approximates the true conditional distribution $p(z \mid \tauexp)$, which makes this bound tight.
In addition, we do not have access to the problem $\mdpindex$ at test time and hence cannot sample from $F_\paramF(z \mid \mdpindex)$.
Therefore, $q_\paramdecoder$ serves as a \emph{decoder} that generates the encoding $z$ from the exploration trajectory $\tauexp$. 

Recall, our goal is to maximize \refeqn{obj-exp} w.r.t., trajectories $\tauexp$ from the exploration policy $\piexp$.
Only the third term depends on $\tauexp$, so we train $\piexp$ on rewards set to be this third term:
\begin{align}
\label{eq:reward-exp}
    &\rexp_t(a_t, r_t, s_{t + 1}, \tauexp_{:t}; \mdpindex) = \\
    &\;\; \E_{z \sim F_\paramF(z \mid \mdpindex)} \left[\log\frac{q_\paramdecoder(z \mid \tauexp_{:t + 1} = [s_{t + 1}; a_t; r_t; \tauexp_{:t}])}{q_\paramdecoder(z \mid \tauexp_{:t})}\right] - \rewardpenalty. \nonumber
\end{align}
Intuitively, the exploration reward for taking action $a_t$ and transitioning to state $s_{t + 1}$ is high if this transition encodes more information about the problem (and hence the encoding $z \sim F_\paramF(z \mid \mdpindex)$) than was already present in the trajectory $\tauexp_{:t}$.
In other words, the reward is the information gain on $z$ from observing $(a_t, r_t, s_{t + 1})$.
We also include a small penalty $\rewardpenalty$ to encourage exploring efficiently in as few timesteps as possible.
This reward is attractive because (i) it is independent from the exploitation policy and hence avoids the local optima described in Section~\ref{sec:cyclic}, and (ii) it is dense, so it helps with credit assignment.
It is also non-Markov, since it depends on $\tauexp$, so we maximize it with a recurrent $\piexp_\paramexp(a_t \mid s_t, \tauexp_{:t})$, parametrized by $\paramexp$.

\subsection{A Practical Implementation of \ours}\label{sec:practical_dream}
\setlength{\textfloatsep}{15pt}\begin{algorithm}[t]
    \small
    \begin{flushleft}
    \begin{algorithmic}[1]
        \State Sample a problem $\mdpindex \sim p(\mdpindex)$
        \State Compute problem ID encoding $z \sim F_\paramF(z \mid \mdpindex)$
        
        \item[]
        \State{\textcolor{blue}{// Exploration episode}}
        \State Roll out exploration policy $\tauexp \sim \piexp_\paramexp(a_t \mid s_t, \tauexp_{:t})$
        \State Update $\piexp_\paramexp$ and $q_\paramdecoder$ to maximize $I(\tauexp; z)$ via rewards in (\ref{eq:reward-exp})
        
        \item[]
        \State{\textcolor{blue}{// Exploitation episode}}
        \State \cameraupdate{Every other episode, choose $z \sim q_\paramdecoder(z \mid \tauexp)$}
        \State Roll out exploitation policy $\piin_\paramin(a \mid s, z)$
        \State Update $\piin_\paramin$ and $F_\paramF$ to maximize (\ref{eq:obj-encoding})
    \end{algorithmic}
    \end{flushleft}
    \caption{\ours meta-training trial}
    \label{alg:ours-simple}
\end{algorithm}

\cameraupdate{Altogether, \ours learns four components.
We summarize each component and detail practical choices for parametrizing them as neural networks below.
}

\vspace{-2mm}
\begin{enumerate}[leftmargin=*]
\item \underline{Encoder $F_\paramF(z \mid \mdpindex)$}:
  \cameraupdate{The encoder learns to extract only task-relevant information from the problem ID $\mdpindex$ via Equation \ref{eq:obj-encoding}.
  Then, \ours learns to efficiently explore by recovering the extracted information.}
  For simplicity, we parametrize the stochastic encoder by learning a deterministic encoding $f_\paramF(\mdpindex)$ and apply Gaussian noise, i.e., $F_\paramF(z \mid \mdpindex) = \mathcal{N}(f_\paramF(\mdpindex), \rho^2 I)$.
We choose a convenient prior $j(z)$ to be a unit Gaussian with same variance $\rho^2 I$, which makes the information bottleneck take the form of simple $\ell_2$-regularization $I(z; \mdpindex) = \| f_\paramF(\mdpindex) \|_2^2$.
    
  \item \underline{Decoder $q_\paramdecoder(z \mid \tauexp)$}:
  \cameraupdate{The decoder learns to map exploration trajectories $\tauexp$ to encodings $z$, used by the exploitation policy during meta-test time, via maximizing Equation \ref{eqn:obj-exp}.}
  Similar to the encoder, we parametrize the decoder $q_\paramdecoder(z \mid \tauexp)$ as a Gaussian  centered around a deterministic encoding $g_\paramdecoder(\tauexp)$ with variance $\rho^2 I$.
    Then, $q_\paramdecoder$ minimizes $\E_{\mdpindex, z \sim F_\paramF(z \mid \mdpindex)}\left[\norm{z - g_\paramdecoder(\tauexp)}^2_2\right]$ w.r.t., $\paramdecoder$ (Equation~\ref{eqn:obj-exp}),
    and the exploration rewards take the form:\\
{$\thickmuskip=5mu \rexp(a, r, s', \tauexp; \mdpindex) = -\norm{f_\paramF(\mdpindex) - g_\paramdecoder([\tauexp; a; r; s'])}^2_2 + \norm{f_\paramF(\mdpindex) - g_\paramdecoder(\tauexp)}^2_2 - \rewardpenalty$} (Equation~\ref{eq:reward-exp}).

  \item \underline{Exploitation policy $\piin_\paramin$} and 4. \underline{Exploration policy $\piexp_\paramexp$}: We learn both policies with double deep Q-learning~\citep{van2016deep}, treating $(s, z)$ as the state for $\piin_\paramin$.
\end{enumerate}
\vspace{-2mm}
In practice, we jointly learn all components by following \refalg{ours-simple} each meta-training trial.
Overall, this avoids the chicken-and-egg problem in \refsec{cyclic} by learning exploitation and the encoder (lines 6--9) independently from exploration.
This enables the encoder to learn quickly, and once it is learned, it forms a learning signal for exploration separate from the expected exploitation returns (lines 3--5), which improves sample efficiency (\refsec{sample_complexity}).

During meta-testing, $\mdpindex$ is unavailable, but since $\piexp_\paramexp$ learns to produce exploration trajectories $\tauexp$ containing the same information as $z \sim F_\paramF(z \mid \mdpindex)$,
we can generate $z$ from $q_\paramdecoder(z \mid \tauexp)$ instead of from $F_\paramF(z \mid \mdpindex)$ for the exploitation policy $\piin_\paramin(a \mid s, z)$.
\cameraupdate{Since the exploitation policy conditions on $z \sim q_\paramdecoder(z \mid \tauexp)$ from the decoder during meta-testing, we also train the exploitation policy conditioned on $z \sim q_\paramdecoder(z \mid \tauexp)$ every other episode during meta-training (line 7), which improves stability.
See \refapp{training_details} for detailed pseudocode and other training details.
}

\begin{figure*}
\center
\vspace{-2mm}
\includegraphics[width=\textwidth]{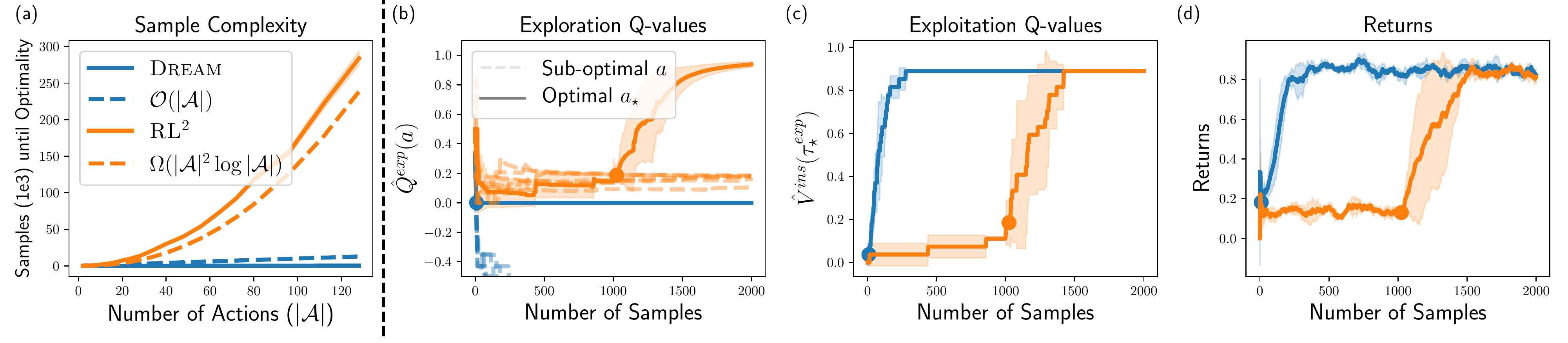}
\vspace*{-7mm}
\caption{
    \small
    (a) Sample complexity of learning the optimal exploration policy as the action space $|\sA|$ grows (1000 seeds).
    (b) Exploration Q-values $\qexp(a)$. The policy $\arg\max_a \qexp(a)$ is optimal after the dot.
    (c) Exploitation values given optimal trajectory $\vin(\tauexp_\star)$.
    (d) Returns achieved on a tabular MDP with $|\sA| = 8$ (3 seeds).
}\label{fig:sample_complexity}
\vspace{-2mm}
\end{figure*}

\begin{figure}\center
\includegraphics[width=0.9\columnwidth]{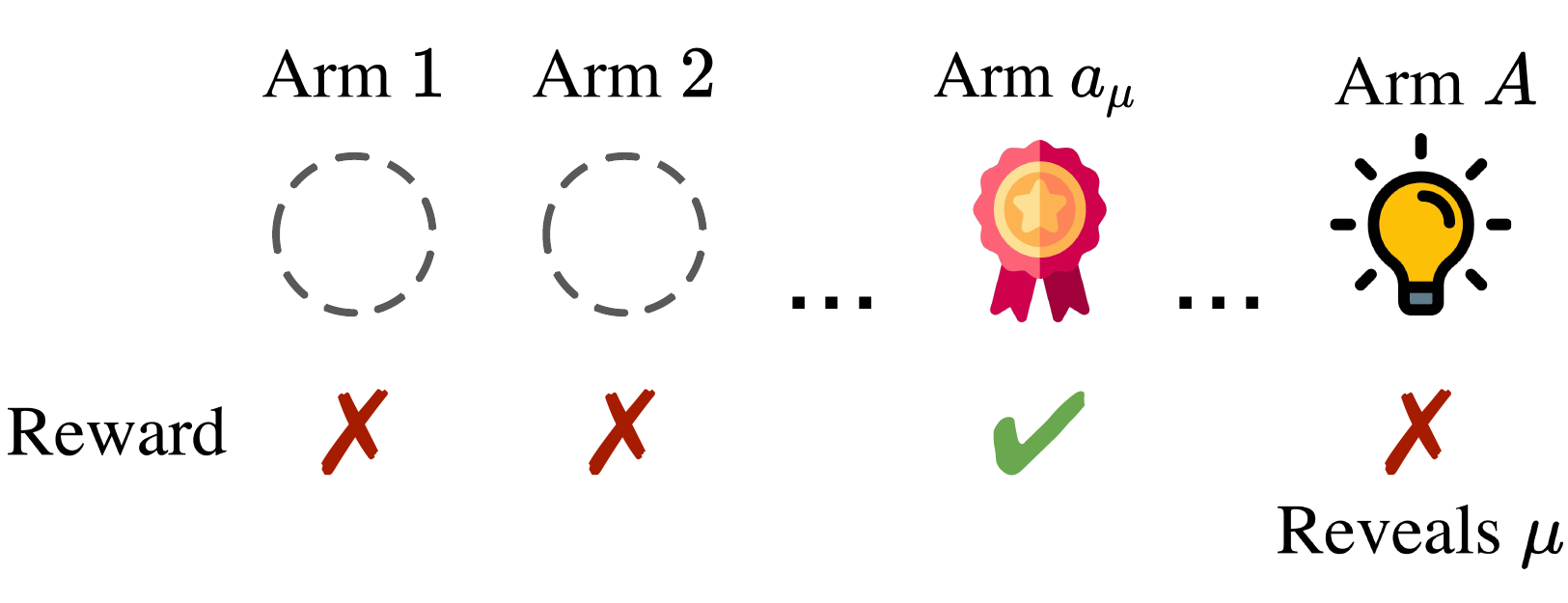}
\vspace{-5mm}
\caption{
    \cameraupdate{\small Simple bandit domain.
    In problem ID $\mdpindex$, action $a_\mdpindex$ obtains reward $1$; all other actions obtain no reward.
    In all problems,
    action $a_\star = A$ reveals $\mdpindex$ and hence, is optimal for exploration.}
}\label{fig:toy_illustration}
\vspace{-2mm}
\end{figure}

\section{Analysis of \ours}\label{sec:analysis}
\subsection{Theoretical Consistency of the \ours Objective}\label{sec:consistency}
A key property of \ours is that it is \emph{consistent}:
maximizing our decoupled objective also maximizes expected returns (Equation~\ref{eqn:standard_objective}).
This contrasts prior decoupled approaches~\citep{zhou2019environment, rakelly2019efficient, gurumurthy2019mame, zhang2020learn}, which also decouple exploration from exploitation, but do not recover the optimal policy even with infinite meta-training trials.
Formally,
\begin{restatable}{prop}{consistency}\label{prop:consistency}
Assume $\langle \mathcal{S}, \mathcal{A}, \mathcal{R}_\mdpindex, \mathcal{T}_\mdpindex \rangle$ is ergodic for all problems $\mdpindex \in \problems$.
Let $V^*(\mdpindex)$ be the maximum expected returns achievable by any exploitation policy with access to the problem ID $\mdpindex$, i.e., with complete information.
Let $\piin_\star, \piexp_\star, F_\star$ and $q_\star(z \mid \tauexp)$ be the optimizers of the \ours objective.
Then, if the function classes \ours optimizes over are well-specified, there exists a finite $T$ such that if the length of the exploration episode is at least $T$,
\begin{equation*}
    \E_{\mdpindex \sim p(\mdpindex), \tauexp \sim \piexp_\star, z \sim q_\star(z \mid \tauexp)}\left[V^{\piin_\star}(z; \mdpindex) \right] =
    \E_{\mdpindex \sim p(\mdpindex)}\left[V^{*}(\mdpindex) \right].
\end{equation*}
\end{restatable}
\cameraupdate{At the optimum of \ours's objective, the exploration and exploitation policies together achieve}
the maximal returns $V^*(\mdpindex)$ without access to $\mdpindex$ during meta-testing (proof in \refapp{consistency_proof}).
\cameraupdate{This result depends on an ergodicity assumption, which ensures that it is possible to recover all task-relevant information in a  single exploration episode.
However, this can be easily removed by increasing the number of exploration episodes.
Furthermore, \ours empirically achieves near-optimal returns even on non-ergodic MDPs in the experiments (\refsec{experiments}).
}

\subsection{Illustrating the Effect of Coupling on Sample Complexity}\label{sec:sample_complexity}

With enough capacity, end-to-end approaches can also learn the optimal policy, but can be highly sample inefficient due to the coupling problem in \refsec{cyclic}.
We highlight this in a simple tabular example to remove
the effects of function approximation, \cameraupdate{illustrated in \reffig{toy_illustration}}:
Each episode is a one-step bandit problem with action space $\sA = \{1, \ldots, A\}$.
Taking action \cameraupdate{$a_\star = A$}
in the exploration episode leads to a trajectory $\tauexp_\star$ that reveals the problem ID $\mdpindex$;
all other actions $a$ reveal no information and lead to $\tauexp_a$.
The ID $\mdpindex$ identifies a unique action \cameraupdate{$a_\mdpindex$} that receives reward $1$ during exploitation; all other actions receive reward $0$.
Therefore, taking $a_\star$ during exploration is necessary and sufficient to obtain optimal reward $1$.
We now study the number of samples required for \rl (the canonical end-to-end approach) and \ours to learn the optimal exploration policy with $\epsilon$-greedy tabular Q-learning.
We precisely describe a more general setup in \refapp{tabular_example} and prove that \emph{\ours learns the optimal exploration policy in $\Omega(|\sA|^H |\sM|)$ times fewer samples than \rl} in this simple setting with horizon $H$.
\reffig{sample_complexity}a empirically validates this result and we provide intuition below.

In the tabular analog of \rl, the exploitation Q-values form targets for the exploration Q-values: $\qexp(a) \leftarrow \vin(\tauexp_a) \coloneqq \max_{a'}\qin(\tauexp_a, a')$.
We drop the fixed initial state from notation.
This creates the local optimum in \refsec{cyclic}.
Initially $\vin(\tauexp_\star)$ is low, as the exploitation policy has not learned to achieve reward, even when given $\tauexp_\star$.
This causes $\qexp(a_\star)$ to be small and therefore $\arg\max_{a}\qexp(a) \neq a_\star$ (\reffig{sample_complexity}b),
which then prevents $\vin(\tauexp_\star)$ from learning (\reffig{sample_complexity}c) as $\tauexp_\star$ is roughly sampled only once per $\frac{|\sA|}{\epsilon}$ episodes.
This effect is mitigated only when $\qexp(a_\star)$ becomes higher than $\qexp(a)$ for the other uninformative $a$'s (the dot in \reffig{sample_complexity}b-d).
Then, learning both the exploitation and exploration Q-values accelerates, but getting there takes many samples.

In \ours, the exploration Q-values regress toward the decoder $\hat{q}$:
$\qexp(a) \leftarrow \log{\hat{q}(\mdpindex \mid \tauexp(a))}$.
This decoder learns much faster than $\qin$, since it does not depend on the exploitation actions. 
Consequently, \ours's exploration policy quickly becomes optimal (dot in \reffig{sample_complexity}b), which enables quickly learning the exploitation Q-values and achieving high reward (Figures~\ref{fig:sample_complexity}c~and~\ref{fig:sample_complexity}d).

In general, \ours learns in far fewer samples than end-to-end approaches, since in end-to-end approaches like \rl, exploration is learned from a quantity requiring many samples to accurately estimate (i.e., the exploitation Q-values in this case).
Initially, this quantity is estimated poorly, so the signal for exploration can erroneously "down weight" good exploration behavior, leading to the chicken-and-egg problem.
In contrast, in \ours, the exploration policy learns from the decoder, which requires far fewer samples to accurately estimate, avoiding the chicken-and-egg problem.

 \section{Experiments}\label{sec:experiments}
Many real-world problem distributions (e.g., cooking) require exploration (e.g., locating ingredients) that is distinct from exploitation (e.g., cooking these ingredients).
Therefore, we desire benchmarks that require distinct exploration and exploitation to stress test aspects of exploration in meta-RL, such as if methods can:
(i) efficiently explore, even in the presence of distractions;
(ii) leverage informative objects (e.g., a map) to aid exploration;
(iii) learn exploration and exploitation strategies that generalize to unseen problems;
(iv) scale to challenging exploration problems with high-dimensional visual observations.
Existing benchmarks (e.g., MetaWorld~\citep{yu2019meta} or MuJoCo tasks like HalfCheetahVelocity~\citep{finn2017modelagnostic, rothfuss2018promp}) were not designed to test exploration and are unsuitable for answering these questions.
These benchmarks mainly vary the rewards (e.g., the speed to run at) across problems,
so naively exploring by exhaustively trying different exploitation behaviors (e.g., running at different speeds) is optimal.
They further don't include visual states, distractors, or informative objects,
which test if exploration is efficient.
We therefore design new benchmarks meeting the above criteria, testing (i-iii) with didactic benchmarks, and
(iv) with a sparse-reward 3D visual navigation benchmark, based on~\citet{kamienny2020learning}, that combines complex exploration with high-dimensional visual inputs.
To further deepen the exploration challenge, we make our benchmarks goal-conditioned.
This requires exploring to discover information relevant to \emph{any} potential goal, rather than just a single task (e.g., locating all ingredients for \emph{any} meal vs. just the ingredients for pasta).

\begin{figure}[t]\center
    \includegraphics[width=0.75\columnwidth]{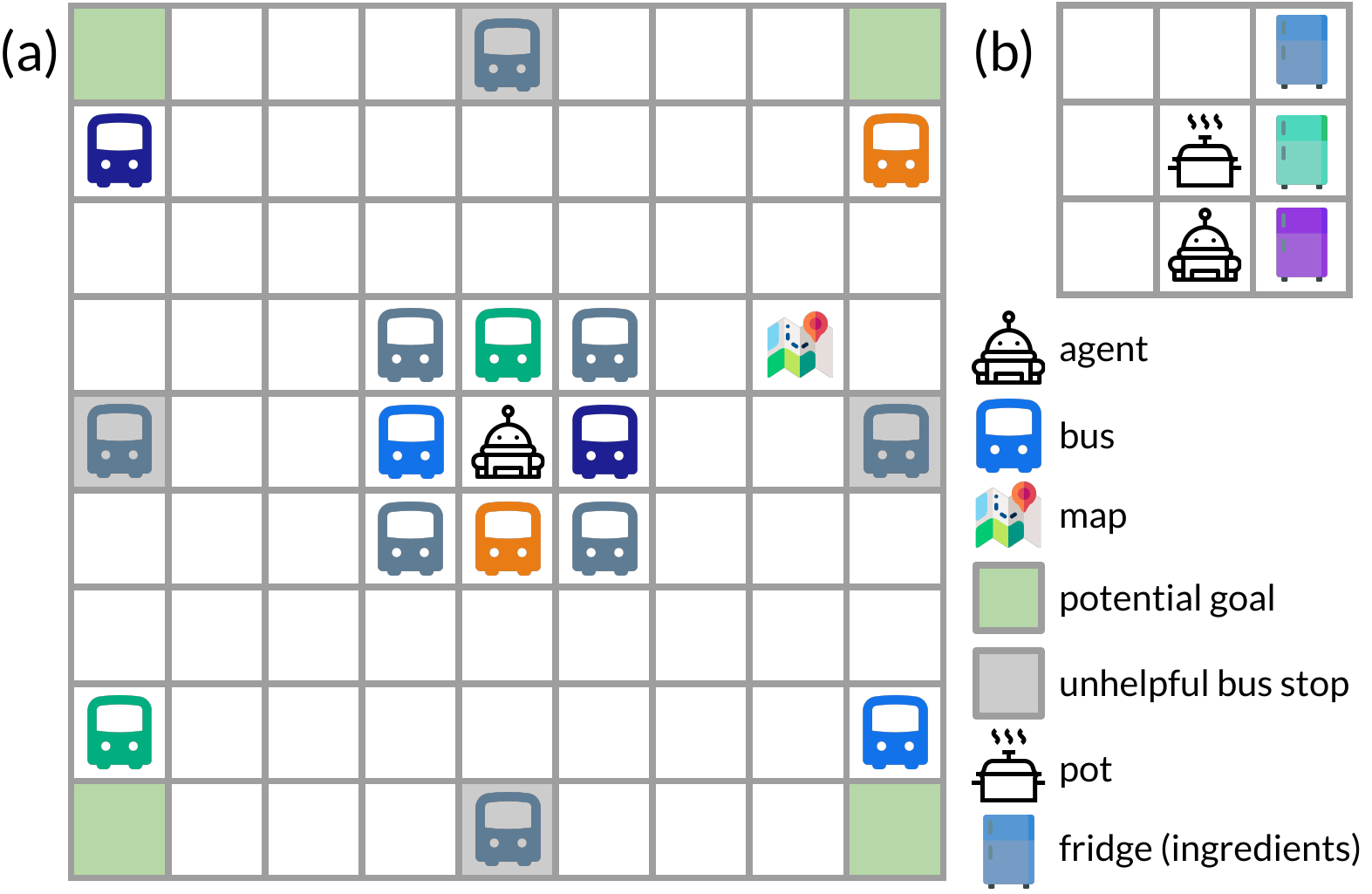}
    \vspace{-3mm}
    \caption{
        Didactic grid worlds to stress test exploration.
        (a) Navigation.
        (b) Cooking.
}\label{fig:gridworld}
    \vspace{-3mm}
\end{figure}

\textbf{Comparisons.} We compare \ours with state-of-the-art end-to-end (E-\rl~\citep{stadie2018importance}, \varibad~\citep{zintgraf2019varibad}, and \import~\citep{kamienny2020learning})
and decoupled approaches (\pearlub, an upper bound on the final performance of \pearl~\citep{rakelly2019efficient}).
For \pearlub, we analytically compute the expected rewards achieved by optimal Thompson sampling (TS) exploration, assuming access to the optimal problem-specific policy and true posterior problem distribution.
Like \ours, \import and \pearl also use the one-hot problem ID, during meta-training.
We also report the optimal returns achievable with no exploration as "No exploration."
Where applicable, all methods use the same architecture.
The full architecture and approach details are in \refapp{baseline_details}.

We report the average returns achieved by each approach in trials with one exploration and one exploitation episode,
averaged over 3 seeds with 1-standard deviation error bars (full details in \refapp{experiment_details}).
We evaluate each approach on 100 meta-testing trials, every 2K meta-training trials.
In all plots, the training timesteps includes all timesteps from both exploitation and exploration episodes in meta-training trials.

\subsection{Didactic Experiments}

\begin{figure}[t]\center
  \vspace{-3mm}
  \includegraphics[width=0.75\columnwidth]{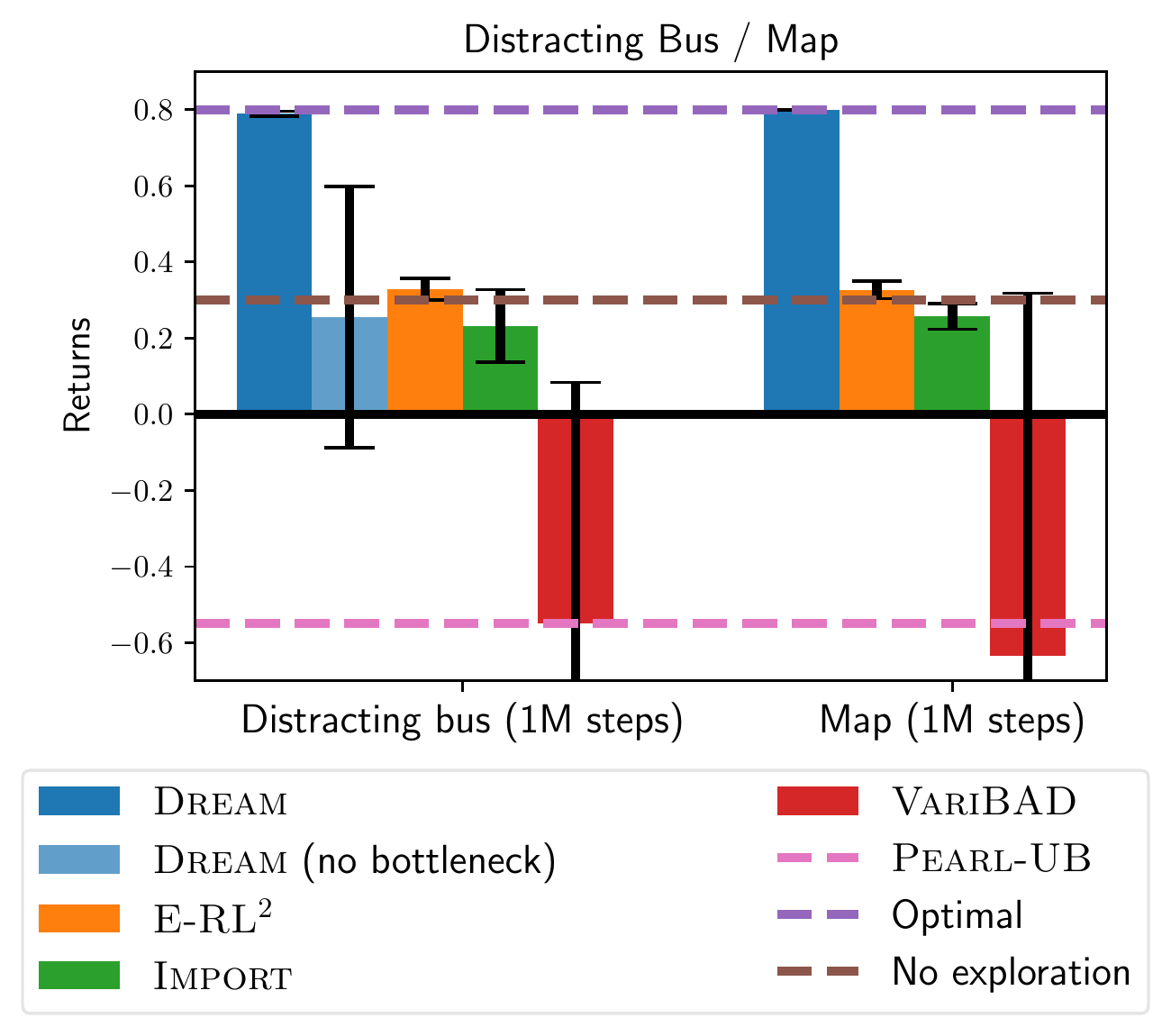}
\vspace{-5mm}
  \caption{
    \small
    Navigation results.
    Only \ours optimally explores all buses and the map.
  }
  \vspace{-1mm}
  \label{fig:distraction_results}
\end{figure}

We first evaluate on the grid worlds shown in Figures~\ref{fig:gridworld}a~and~\ref{fig:gridworld}b.
The state consists of the agent's $(x, y)$-position,
a one-hot indicator of the object at the agent's position (none, bus, map, pot, or fridge),
a one-hot indicator of the agent's inventory (none or an ingredient),
and the goal.
The actions are \emph{move} up, down, left, or right;
\emph{ride bus}, which, at a bus, teleports the agent to another bus of the same color;
\emph{pick up}, which, at a fridge, fills the agent's inventory with the fridge's ingredients;
and \emph{drop}, which, at the pot, empties the agent's inventory into the pot.
Episodes consist of 20 timesteps and the agent receives $-0.1$ reward at each timestep until the goal, described below, is met (details in \refapp{task_details}; qualitative results in \refapp{additional_results}).

\begin{figure*}[h]\center
\vspace{-2mm}
\includegraphics[width=\textwidth]{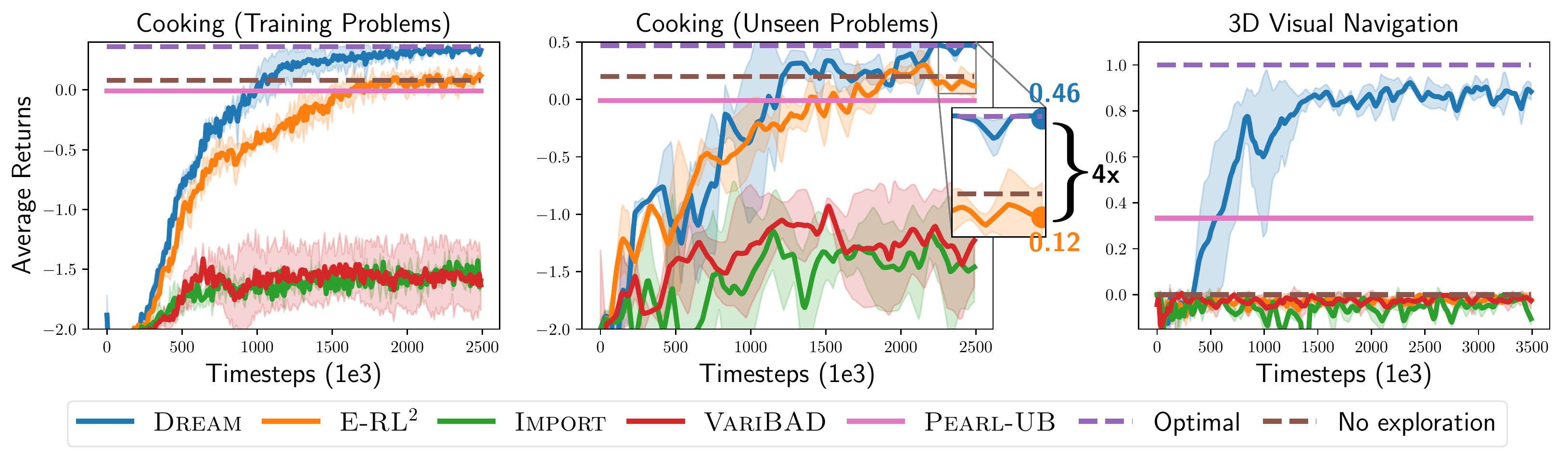}
\vspace{-9mm}
\caption{
  Cooking results:
  only \ours achieves optimal reward on training problems (left),
  and on generalizing to unseen problems (middle).
  3D visual navigation results: only \ours reads the sign and solves the task (right).
}\label{fig:cooking_results}
\vspace{-3mm}
\end{figure*}

\textbf{Targeted exploration.}
We first test if these methods can efficiently explore in the presence of distractions in two versions of the benchmark in \reffig{gridworld}a: \emph{distracting bus} and \emph{map}.
In both, there are 4 possible goals (the 4 green locations).
During each episode, a goal is randomly sampled.
Reaching the goal yields +1 reward and ends the episode.
The 4 colored buses each lead to near a different potential green goal location when ridden and
in different problems $\mdpindex$, their destinations are set to be 1 of the $4!$ different permutations.
The \emph{distracting bus} version tests if the agent can ignore distractions by including unhelpful gray buses,
which are never needed to optimally reach any goal.
In different problems, the gray buses lead to different permutations of the gray locations.
The \emph{map} version tests if the agent can leverage objects for exploration by including a map that reveals the destinations of the colored buses when touched.

Figure \ref{fig:distraction_results} shows the results after 1M steps.
\ours learns to optimally explore and thus receives optimal reward in both versions:
In \emph{distracting bus}, it ignores the unhelpful gray buses and learns the destinations of all helpful buses by riding them.
In \emph{map}, it learns to leverage informative objects by visiting the map.
During exploitation, \ours immediately reaches the goal by riding the correct colored bus.
In contrast, \import and E-\rl get stuck in a local optimum, indicative of the coupling problem (\refsec{cyclic}), which achieves the same returns as no exploration at all.
They do not explore the helpful buses or map and consequently sub-optimally exploit by just walking to the goal.
\varibad learns slower, likely because it learns a dynamics model, but eventually matches the sub-optimal returns of \import and \rl in \textasciitilde 3M steps (not shown).

\pearl achieves sub-optimal returns, even with infinite meta-training (see line for \pearlub), as follows.
TS explores by sampling a problem ID from its posterior and executing its policy conditioned on this ID.
Since for any given problem (bus configuration) and goal, the optimal problem-specific policy rides the one bus leading to the goal, TS does not explore optimally (i.e., explore all the buses or read the map), even with the optimal problem-specific policy and true posterior problem distribution.

Recall that \ours tries to discard extraneous task-irrelevant information from the problem ID with an information bottleneck that minimizes the mutual information $I(z; \mdpindex)$ between problem IDs and the encoder $F_\paramF(z \mid \mdpindex)$.
\cameraupdate{This makes exploration targeted, since \ours only explores to recover information in $z$.
We hypothesize that the bottleneck only improves exploration in domains with distracting task-irrelevant information to discard from the problem ID.
This empirically holds when we ablate the information bottleneck from \ours, plotted under \ours (no bottleneck):
In \emph{distracting bus}, \ours without the bottleneck wastes its exploration on the distracting gray unhelpful buses and consequently achieves low returns, as seen in \reffig{distraction_results} (left).
In contrast, \emph{map} and other below domains do not contain any distracting information in the problem ID.
Consistent with our hypothesis, \ours achieves comparable returns with or without the information bottleneck in these domains, as seen in \reffig{distraction_results} (right) and \reffig{cooking_results}.
}

\textbf{Generalization to new problems.}
We test generalization to unseen problems in a cooking benchmark (\reffig{gridworld}b).
The fridges each contain 1 of 4 different (color-coded) ingredients, determined by the problem ID.
The fridges' contents are unobserved until the agent uses the "pick up" action at the fridge.
Goals (recipes) specify placing 2 correct ingredients in the pot in the right order.
The agent receives positive reward for picking up and placing the correct ingredients, and negative reward for using wrong ingredients.
We hold out 1 of the $4^3 = 64$ problems for meta-testing.

\reffig{cooking_results} shows the results on training (left) and held-out (middle) problems.
Only \ours achieves near-optimal returns on both.
During exploration, it investigates each fridge with the "pick up" action, and then directly retrieves the correct ingredients during exploitation.
E-\rl gets stuck in a local optimum, only sometimes exploring the fridges.
This achieves 3.8x lower returns, only slightly higher than no exploration at all.
Here, leveraging the problem ID actually hurts \import compared to E-\rl.
\import successfully solves the task, given access to the problem ID, but fails without it.
As before, \varibad learns slowly and TS (\pearlub) cannot learn optimal exploration.

\subsection{Sparse-Reward 3D Visual Navigation}

\begin{figure}\center
    \includegraphics[width=0.75\columnwidth]{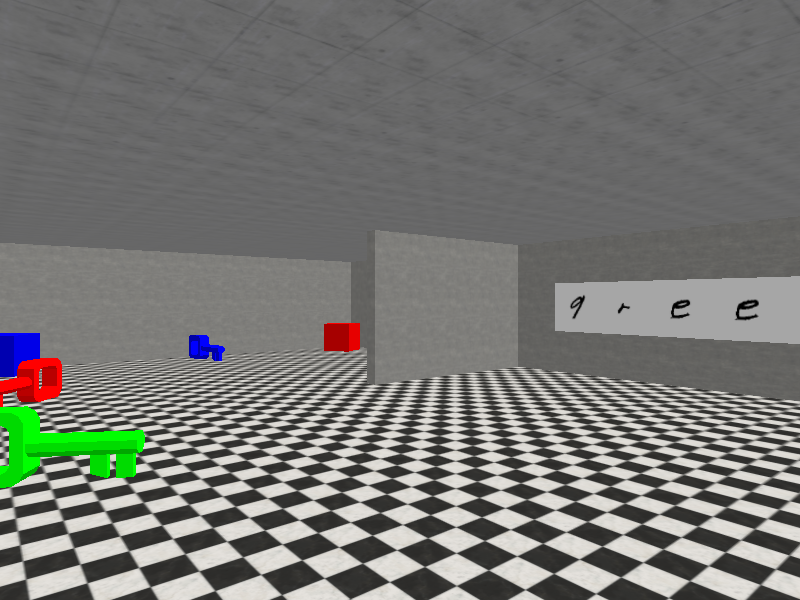}
    \vspace{-1mm}
    \caption{
        \small
        3D Visual Navigation. The agent must read the sign to determine what colored object to go to.
    }\label{fig:sign}
    \vspace{-4mm}
\end{figure}

We conclude with a challenging benchmark testing both sophisticated exploration and scalability to pixel inputs.
We modify a benchmark from~\citet{kamienny2020learning} to increase both the exploration and scalability challenge by including more objects and a visual sign, illustrated in \reffig{sign}.
In the 3 different problems, the sign on the right says ``blue'', ``red'' or ``green.''
The goals specify whether the agent should collect the key or block.
The agent receives +1 reward for collecting the correct object (color specified by the sign, shape specified by the goal), -1 reward for the wrong object, and 0 reward otherwise.
The agent begins the episode on the far side of the barrier and must walk around the barrier to visually ``read'' the sign.
The agent's observations are $80 \times 60$ RGB images and its actions are to rotate left or right, move forward, or end the episode.

\ours is the only method that learns to read the sign and achieve reward (\reffig{cooking_results} right).
All end-to-end approaches get stuck in local optima due to the chicken-and-egg coupling problem, where they do not learn to read the sign and hence stay away from all the objects, in fear of receiving negative reward.
This achieves close to $0$ returns, consistent with the results in \citet{kamienny2020learning}.
As before, \pearlub cannot learn optimal exploration.
 \section{Conclusion}

In summary, this work identifies a chicken-and-egg problem that end-to-end meta-RL approaches suffer from,
where learning good exploitation requires already having learned good exploration and vice-versa.
This creates challenging local optima, since typically neither exploration nor exploitation is good at the beginning of meta-training.
We show that appropriately leveraging simple one-hot problem IDs allows us to break this cyclic dependency with \ours.
Consequently, \ours has strong empirical performance on meta-RL problems requiring complex exploration, as well as substantial theoretical sample complexity improvements in the tabular setting.
Though prior works also leverage the problem ID and use decoupled objectives that avoid the chicken-and-egg problem, no other existing approaches can recover optimal exploration empirically and theoretically like \ours.

\textbf{Reproducibility.} Our code is publicly available at \url{https://github.com/ezliu/dream}.

 \section*{Acknowledgements}
We thank Luisa Zintgraf for her insights about \varibad.
We also thank Sahaana Suri, Suraj Nair, Minae Kwon, and Ramtin Keramati for their feedback on drafts of this paper.

We thank Arkira Chantaratananond for creating animations of the grid world tasks in the project web page.
Icons used in this paper were made by Freepik, ThoseIcons, dDara, Pixel perfect, ThoseIcons, mynamepong, Icongeek26, and Vitaly Gorbachev from \url{www.flaticon.com}.

EL is supported by a National Science Foundation Graduate Research Fellowship under Grant No. DGE-1656518.
AR is supported by a Google PhD Fellowship and Open Philanthropy Project AI Fellowship.
This work was also supported in part by Google.

\bibliography{all}

\begin{thebibliography}{56}
\providecommand{\natexlab}[1]{#1}
\providecommand{\url}[1]{\texttt{#1}}
\expandafter\ifx\csname urlstyle\endcsname\relax
  \providecommand{\doi}[1]{doi: #1}\else
  \providecommand{\doi}{doi: \begingroup \urlstyle{rm}\Url}\fi

\bibitem[Agarwal et~al.(2019)Agarwal, Liang, Schuurmans, and
  Norouzi]{agarwal2019learning}
Agarwal, R., Liang, C., Schuurmans, D., and Norouzi, M.
\newblock Learning to generalize from sparse and underspecified rewards.
\newblock \emph{arXiv preprint arXiv:1902.07198}, 2019.

\bibitem[Alemi et~al.(2016)Alemi, Fischer, Dillon, and Murphy]{alemi2016deep}
Alemi, A.~A., Fischer, I., Dillon, J.~V., and Murphy, K.
\newblock Deep variational information bottleneck.
\newblock \emph{arXiv preprint arXiv:1612.00410}, 2016.

\bibitem[Andrychowicz et~al.(2016)Andrychowicz, Denil, Gomez, Hoffman, Pfau,
  Schaul, Shillingford, and Freitas]{andrychowicz2016learning}
Andrychowicz, M., Denil, M., Gomez, S., Hoffman, M.~W., Pfau, D., Schaul, T.,
  Shillingford, B., and Freitas, N.~D.
\newblock Learning to learn by gradient descent by gradient descent.
\newblock In \emph{Advances in neural information processing systems}, pp.\
  3981--3989, 2016.

\bibitem[Barber \& Agakov(2003)Barber and Agakov]{barber2003algorithm}
Barber, D. and Agakov, F.~V.
\newblock The {IM} algorithm: a variational approach to information
  maximization.
\newblock In \emph{Advances in neural information processing systems}, 2003.

\bibitem[Bellemare et~al.(2016)Bellemare, Srinivasan, Ostrovski, Schaul,
  Saxton, and Munos]{bellemare2016unifying}
Bellemare, M., Srinivasan, S., Ostrovski, G., Schaul, T., Saxton, D., and
  Munos, R.
\newblock Unifying count-based exploration and intrinsic motivation.
\newblock In \emph{Advances in Neural Information Processing Systems
  (NeurIPS)}, pp.\  1471--1479, 2016.

\bibitem[Bengio et~al.(1992)Bengio, Bengio, Cloutier, and
  Gecsei]{bengio1992optimization}
Bengio, S., Bengio, Y., Cloutier, J., and Gecsei, J.
\newblock On the optimization of a synaptic learning rule.
\newblock In \emph{Preprints Conf. Optimality in Artificial and Biological
  Neural Networks}, volume~2, 1992.

\bibitem[Bengio et~al.(1991)Bengio, Bengio, and Cloutier]{bengio1991learning}
Bengio, Y., Bengio, S., and Cloutier, J.
\newblock Learning a synaptic learning rule.
\newblock In \emph{IJCNN-91-Seattle International Joint Conference on Neural
  Networks}, volume~2, pp.\  969--969, 1991.

\bibitem[Burda et~al.(2018)Burda, Edwards, Storkey, and
  Klimov]{burda2018exploration}
Burda, Y., Edwards, H., Storkey, A., and Klimov, O.
\newblock Exploration by random network distillation.
\newblock \emph{arXiv preprint arXiv:1810.12894}, 2018.

\bibitem[Chevalier-Boisvert(2018)]{boisvert2018gym}
Chevalier-Boisvert, M.
\newblock Gym-{M}iniworld environment for openai gym.
\newblock \url{https://github.com/maximecb/gym-miniworld}, 2018.

\bibitem[Dorfman \& Tamar(2020)Dorfman and Tamar]{dorfman2020offline}
Dorfman, R. and Tamar, A.
\newblock Offline meta reinforcement learning.
\newblock \emph{arXiv preprint arXiv:2008.02598}, 2020.

\bibitem[Duan et~al.(2016)Duan, Schulman, Chen, Bartlett, Sutskever, and
  Abbeel]{duan2016rl}
Duan, Y., Schulman, J., Chen, X., Bartlett, P.~L., Sutskever, I., and Abbeel,
  P.
\newblock {RL}$^2$: Fast reinforcement learning via slow reinforcement
  learning.
\newblock \emph{arXiv preprint arXiv:1611.02779}, 2016.

\bibitem[Eysenbach et~al.(2018)Eysenbach, Gupta, Ibarz, and
  Levine]{eysenbach2018diversity}
Eysenbach, B., Gupta, A., Ibarz, J., and Levine, S.
\newblock Diversity is all you need: Learning skills without a reward function.
\newblock \emph{arXiv preprint arXiv:1802.06070}, 2018.

\bibitem[Fakoor et~al.(2019)Fakoor, Chaudhari, Soatto, and
  Smola]{fakoor2019meta}
Fakoor, R., Chaudhari, P., Soatto, S., and Smola, A.~J.
\newblock Meta-{Q}-learning.
\newblock \emph{arXiv preprint arXiv:1910.00125}, 2019.

\bibitem[Finn et~al.(2017)Finn, Abbeel, and Levine]{finn2017modelagnostic}
Finn, C., Abbeel, P., and Levine, S.
\newblock Model-agnostic meta-learning for fast adaptation of deep networks.
\newblock In \emph{International Conference on Machine Learning (ICML)}, 2017.

\bibitem[Gregor et~al.(2016)Gregor, Rezende, and
  Wierstra]{gregor2016variational}
Gregor, K., Rezende, D.~J., and Wierstra, D.
\newblock Variational intrinsic control.
\newblock \emph{arXiv preprint arXiv:1611.07507}, 2016.

\bibitem[Gupta et~al.(2018)Gupta, Mendonca, Liu, Abbeel, and
  Levine]{gupta2018meta}
Gupta, A., Mendonca, R., Liu, Y., Abbeel, P., and Levine, S.
\newblock Meta-reinforcement learning of structured exploration strategies.
\newblock In \emph{Advances in Neural Information Processing Systems
  (NeurIPS)}, pp.\  5302--5311, 2018.

\bibitem[Gurumurthy et~al.(2019)Gurumurthy, Kumar, and
  Sycara]{gurumurthy2019mame}
Gurumurthy, S., Kumar, S., and Sycara, K.
\newblock Mame: Model-agnostic meta-exploration.
\newblock \emph{arXiv preprint arXiv:1911.04024}, 2019.

\bibitem[Hiraoka et~al.(2020)Hiraoka, Imagawa, Tangkaratt, Osa, Onishi, and
  Tsuruoka]{hiraoka2020meta}
Hiraoka, T., Imagawa, T., Tangkaratt, V., Osa, T., Onishi, T., and Tsuruoka, Y.
\newblock Meta-model-based meta-policy optimization.
\newblock \emph{arXiv preprint arXiv:2006.02608}, 2020.

\bibitem[Hochreiter et~al.(2001)Hochreiter, Younger, and
  Conwell]{hochreiter2001learning}
Hochreiter, S., Younger, A.~S., and Conwell, P.~R.
\newblock Learning to learn using gradient descent.
\newblock In \emph{International Conference on Artificial Neural Networks
  (ICANN)}, pp.\  87--94, 2001.

\bibitem[Houthooft et~al.(2016)Houthooft, Chen, Duan, Schulman, Turck, and
  Abbeel]{houthooft2016vime}
Houthooft, R., Chen, X., Duan, Y., Schulman, J., Turck, F.~D., and Abbeel, P.
\newblock Vime: Variational information maximizing exploration.
\newblock In \emph{Advances in Neural Information Processing Systems
  (NeurIPS)}, pp.\  1109--1117, 2016.

\bibitem[Houthooft et~al.(2018)Houthooft, Chen, Isola, Stadie, Wolski, Ho, and
  Abbeel]{houthooft2018evolved}
Houthooft, R., Chen, Y., Isola, P., Stadie, B., Wolski, F., Ho, O.~J., and
  Abbeel, P.
\newblock Evolved policy gradients.
\newblock In \emph{Advances in Neural Information Processing Systems
  (NeurIPS)}, pp.\  5400--5409, 2018.

\bibitem[Humplik et~al.(2019)Humplik, Galashov, Hasenclever, Ortega, Teh, and
  Heess]{humplik2019meta}
Humplik, J., Galashov, A., Hasenclever, L., Ortega, P.~A., Teh, Y.~W., and
  Heess, N.
\newblock Meta reinforcement learning as task inference.
\newblock \emph{arXiv preprint arXiv:1905.06424}, 2019.

\bibitem[Kaelbling et~al.(1998)Kaelbling, Littman, and
  Cassandra]{kaelbling1998planning}
Kaelbling, L.~P., Littman, M.~L., and Cassandra, A.~R.
\newblock Planning and acting in partially observable stochastic domains.
\newblock \emph{Artificial intelligence}, 101\penalty0 (1):\penalty0 99--134,
  1998.

\bibitem[Kamienny et~al.(2020)Kamienny, Pirotta, Lazaric, Lavril, Usunier, and
  Denoyer]{kamienny2020learning}
Kamienny, P.-A., Pirotta, M., Lazaric, A., Lavril, T., Usunier, N., and
  Denoyer, L.
\newblock Learning adaptive exploration strategies in dynamic environments
  through informed policy regularization.
\newblock \emph{arXiv preprint arXiv:2005.02934}, 2020.

\bibitem[Kapturowski et~al.(2019)Kapturowski, Ostrovski, Quan, Munos, and
  Dabney]{kapturowski2019recurrent}
Kapturowski, S., Ostrovski, G., Quan, J., Munos, R., and Dabney, W.
\newblock Recurrent experience replay in distributed reinforcement learning.
\newblock In \emph{International Conference on Learning Representations
  (ICLR)}, 2019.

\bibitem[Kingma \& Ba(2015)Kingma and Ba]{kingma2015adam}
Kingma, D. and Ba, J.
\newblock Adam: A method for stochastic optimization.
\newblock In \emph{International Conference on Learning Representations
  (ICLR)}, 2015.

\bibitem[Leibfried et~al.(2019)Leibfried, Pascual-Diaz, and
  Grau-Moya]{leibfried2019unified}
Leibfried, F., Pascual-Diaz, S., and Grau-Moya, J.
\newblock A unified bellman optimality principle combining reward maximization
  and empowerment.
\newblock In \emph{Advances in Neural Information Processing Systems
  (NeurIPS)}, pp.\  7869--7880, 2019.

\bibitem[Liu et~al.(2020{\natexlab{a}})Liu, Hashemi, Swersky, Ranganathan, and
  Ahn]{liu2020imitation}
Liu, E.~Z., Hashemi, M., Swersky, K., Ranganathan, P., and Ahn, J.
\newblock An imitation learning approach for cache replacement.
\newblock \emph{arXiv preprint arXiv:2006.16239}, 2020{\natexlab{a}}.

\bibitem[Liu et~al.(2020{\natexlab{b}})Liu, Keramati, Seshadri, Guu, Pasupat,
  Brunskill, and Liang]{liu2020learning}
Liu, E.~Z., Keramati, R., Seshadri, S., Guu, K., Pasupat, P., Brunskill, E.,
  and Liang, P.
\newblock Learning abstract models for strategic exploration and fast reward
  transfer.
\newblock \emph{arXiv preprint arXiv:2007.05896}, 2020{\natexlab{b}}.

\bibitem[Mendonca et~al.(2019)Mendonca, Gupta, Kralev, Abbeel, Levine, and
  Finn]{mendonca2019guided}
Mendonca, R., Gupta, A., Kralev, R., Abbeel, P., Levine, S., and Finn, C.
\newblock Guided meta-policy search.
\newblock In \emph{Advances in Neural Information Processing Systems
  (NeurIPS)}, pp.\  9653--9664, 2019.

\bibitem[Mishra et~al.(2017)Mishra, Rohaninejad, Chen, and
  Abbeel]{mishra2017simple}
Mishra, N., Rohaninejad, M., Chen, X., and Abbeel, P.
\newblock A simple neural attentive meta-learner.
\newblock \emph{arXiv preprint arXiv:1707.03141}, 2017.

\bibitem[Mnih et~al.(2015)Mnih, Kavukcuoglu, Silver, Rusu, Veness, Bellemare,
  Graves, Riedmiller, Fidjeland, Ostrovski, et~al.]{mnih2015human}
Mnih, V., Kavukcuoglu, K., Silver, D., Rusu, A.~A., Veness, J., Bellemare,
  M.~G., Graves, A., Riedmiller, M., Fidjeland, A.~K., Ostrovski, G., et~al.
\newblock Human-level control through deep reinforcement learning.
\newblock \emph{Nature}, 518\penalty0 (7540):\penalty0 529--533, 2015.

\bibitem[Nagabandi et~al.(2018)Nagabandi, Clavera, Liu, Fearing, Abbeel,
  Levine, and Finn]{nagabandi2018learning}
Nagabandi, A., Clavera, I., Liu, S., Fearing, R.~S., Abbeel, P., Levine, S.,
  and Finn, C.
\newblock Learning to adapt in dynamic, real-world environments through
  meta-reinforcement learning.
\newblock \emph{arXiv preprint arXiv:1803.11347}, 2018.

\bibitem[Naik \& Mammone(1992)Naik and Mammone]{naik1992meta}
Naik, D.~K. and Mammone, R.~J.
\newblock Meta-neural networks that learn by learning.
\newblock In \emph{[Proceedings 1992] IJCNN International Joint Conference on
  Neural Networks}, volume~1, pp.\  437--442, 1992.

\bibitem[Paszke et~al.(2017)Paszke, Gross, Chintala, Chanan, Yang, DeVito, Lin,
  Desmaison, Antiga, and Lerer]{paszke2017automatic}
Paszke, A., Gross, S., Chintala, S., Chanan, G., Yang, E., DeVito, Z., Lin, Z.,
  Desmaison, A., Antiga, L., and Lerer, A.
\newblock Automatic differentiation in pytorch, 2017.

\bibitem[Pathak et~al.(2017)Pathak, Agrawal, Efros, and
  Darrell]{pathak2017curiosity}
Pathak, D., Agrawal, P., Efros, A.~A., and Darrell, T.
\newblock Curiosity-driven exploration by self-supervised prediction.
\newblock In \emph{Computer Vision and Pattern Recognition (CVPR)}, pp.\
  16--17, 2017.

\bibitem[Rakelly et~al.(2019)Rakelly, Zhou, Quillen, Finn, and
  Levine]{rakelly2019efficient}
Rakelly, K., Zhou, A., Quillen, D., Finn, C., and Levine, S.
\newblock Efficient off-policy meta-reinforcement learning via probabilistic
  context variables.
\newblock \emph{arXiv preprint arXiv:1903.08254}, 2019.

\bibitem[Rothfuss et~al.(2018)Rothfuss, Lee, Clavera, Asfour, and
  Abbeel]{rothfuss2018promp}
Rothfuss, J., Lee, D., Clavera, I., Asfour, T., and Abbeel, P.
\newblock Promp: Proximal meta-policy search.
\newblock \emph{arXiv preprint arXiv:1810.06784}, 2018.

\bibitem[Russo et~al.(2017)Russo, Roy, Kazerouni, Osband, and
  Wen]{russo2017tutorial}
Russo, D., Roy, B.~V., Kazerouni, A., Osband, I., and Wen, Z.
\newblock A tutorial on thompson sampling.
\newblock \emph{arXiv preprint arXiv:1707.02038}, 2017.

\bibitem[S{\ae}mundsson et~al.(2018)S{\ae}mundsson, Hofmann, and
  Deisenroth]{saemundsson2018meta}
S{\ae}mundsson, S., Hofmann, K., and Deisenroth, M.~P.
\newblock Meta reinforcement learning with latent variable gaussian processes.
\newblock \emph{arXiv preprint arXiv:1803.07551}, 2018.

\bibitem[Santoro et~al.(2016)Santoro, Bartunov, Botvinick, Wierstra, and
  Lillicrap]{santoro2016one}
Santoro, A., Bartunov, S., Botvinick, M., Wierstra, D., and Lillicrap, T.
\newblock One-shot learning with memory-augmented neural networks.
\newblock \emph{arXiv preprint arXiv:1605.06065}, 2016.

\bibitem[Schmidhuber(1987)]{schmidhuber1987evolutionary}
Schmidhuber, J.
\newblock \emph{Evolutionary principles in self-referential learning, or on
  learning how to learn: the meta-meta-... hook}.
\newblock PhD thesis, Technische Universit{\"a}t M{\"u}nchen, 1987.

\bibitem[Stadie et~al.(2018)Stadie, Yang, Houthooft, Chen, Duan, Wu, Abbeel,
  and Sutskever]{stadie2018importance}
Stadie, B., Yang, G., Houthooft, R., Chen, P., Duan, Y., Wu, Y., Abbeel, P.,
  and Sutskever, I.
\newblock The importance of sampling inmeta-reinforcement learning.
\newblock In \emph{Advances in Neural Information Processing Systems
  (NeurIPS)}, pp.\  9280--9290, 2018.

\bibitem[Thompson(1933)]{thompson1933likelihood}
Thompson, W.~R.
\newblock On the likelihood that one unknown probability exceeds another in
  view of the evidence of two samples.
\newblock \emph{Biometrika}, 25\penalty0 (3):\penalty0 285--294, 1933.

\bibitem[Thrun \& Pratt(2012)Thrun and Pratt]{thrun2012learning}
Thrun, S. and Pratt, L.
\newblock \emph{Learning to learn}.
\newblock Springer Science \& Business Media Springer Science \& Business
  Media, 2012.

\bibitem[van~der Maaten \& Hinton(2008)van~der Maaten and
  Hinton]{maaten2008visualizing}
van~der Maaten, L. and Hinton, G.
\newblock Visualizing data using {t}-{SNE}.
\newblock \emph{Journal of machine learning research}, 9\penalty0 (0):\penalty0
  2579--2605, 2008.

\bibitem[van Hasselt et~al.(2016)van Hasselt, Guez, and Silver]{van2016deep}
van Hasselt, H., Guez, A., and Silver, D.
\newblock Deep reinforcement learning with double {Q}-learning.
\newblock In \emph{Association for the Advancement of Artificial Intelligence
  (AAAI)}, volume~16, pp.\  2094--2100, 2016.

\bibitem[Wang et~al.(2016{\natexlab{a}})Wang, Kurth-Nelson, Tirumala, Soyer,
  Leibo, Munos, Blundell, Kumaran, and Botvinick]{wang2016learning}
Wang, J.~X., Kurth-Nelson, Z., Tirumala, D., Soyer, H., Leibo, J.~Z., Munos,
  R., Blundell, C., Kumaran, D., and Botvinick, M.
\newblock Learning to reinforcement learn.
\newblock \emph{arXiv preprint arXiv:1611.05763}, 2016{\natexlab{a}}.

\bibitem[Wang et~al.(2016{\natexlab{b}})Wang, Schaul, Hessel, Hasselt, Lanctot,
  and Freitas]{wang2016dueling}
Wang, Z., Schaul, T., Hessel, M., Hasselt, H.~V., Lanctot, M., and Freitas,
  N.~D.
\newblock Dueling network architectures for deep reinforcement learning.
\newblock In \emph{International Conference on Machine Learning (ICML)},
  2016{\natexlab{b}}.

\bibitem[Warde-Farley et~al.(2018)Warde-Farley, de~Wiele, Kulkarni, Ionescu,
  Hansen, and Mnih]{warde2018unsupervised}
Warde-Farley, D., de~Wiele, T.~V., Kulkarni, T., Ionescu, C., Hansen, S., and
  Mnih, V.
\newblock Unsupervised control through non-parametric discriminative rewards.
\newblock \emph{arXiv preprint arXiv:1811.11359}, 2018.

\bibitem[Yang et~al.(2019)Yang, Caluwaerts, Iscen, Tan, and
  Finn]{yang2019norml}
Yang, Y., Caluwaerts, K., Iscen, A., Tan, J., and Finn, C.
\newblock Norml: No-reward meta learning.
\newblock In \emph{Proceedings of the 18th International Conference on
  Autonomous Agents and MultiAgent Systems}, pp.\  323--331, 2019.

\bibitem[Yu et~al.(2019)Yu, Quillen, He, Julian, Hausman, Finn, and
  Levine]{yu2019meta}
Yu, T., Quillen, D., He, Z., Julian, R., Hausman, K., Finn, C., and Levine, S.
\newblock Meta-world: A benchmark and evaluation for multi-task and meta
  reinforcement learning.
\newblock \emph{arXiv preprint arXiv:1910.10897}, 2019.

\bibitem[Zhang et~al.(2020)Zhang, Wang, Hu, Chen, Fan, and
  Zhang]{zhang2020learn}
Zhang, J., Wang, J., Hu, H., Chen, Y., Fan, C., and Zhang, C.
\newblock Learn to effectively explore in context-based meta-{RL}.
\newblock \emph{arXiv preprint arXiv:2006.08170}, 2020.

\bibitem[Zhou et~al.(2019{\natexlab{a}})Zhou, Jang, Kappler, Herzog, Khansari,
  Wohlhart, Bai, Kalakrishnan, Levine, and Finn]{zhou2019watch}
Zhou, A., Jang, E., Kappler, D., Herzog, A., Khansari, M., Wohlhart, P., Bai,
  Y., Kalakrishnan, M., Levine, S., and Finn, C.
\newblock Watch, try, learn: Meta-learning from demonstrations and reward.
\newblock \emph{arXiv preprint arXiv:1906.03352}, 2019{\natexlab{a}}.

\bibitem[Zhou et~al.(2019{\natexlab{b}})Zhou, Pinto, and
  Gupta]{zhou2019environment}
Zhou, W., Pinto, L., and Gupta, A.
\newblock Environment probing interaction policies.
\newblock \emph{arXiv preprint arXiv:1907.11740}, 2019{\natexlab{b}}.

\bibitem[Zintgraf et~al.(2019)Zintgraf, Shiarlis, Igl, Schulze, Gal, Hofmann,
  and Whiteson]{zintgraf2019varibad}
Zintgraf, L., Shiarlis, K., Igl, M., Schulze, S., Gal, Y., Hofmann, K., and
  Whiteson, S.
\newblock Varibad: A very good method for bayes-adaptive deep {RL} via
  meta-learning.
\newblock \emph{arXiv preprint arXiv:1910.08348}, 2019.

\end{thebibliography}
\bibliographystyle{icml2021}

\clearpage
\appendix

\section{\ours Training Details}\label{sec:training_details}

\refalg{ours} summarizes a practical algorithm for training \ours.
We parametrize
\iclrupdate{both the exploration and exploitation policies as recurrent
}deep dueling double-Q networks~\citep{wang2016dueling, van2016deep},
with exploration Q-values $\qexp(s, \tauexp, a; \paramexp)$ parametrized by $\paramexp$ (and target network parameters $\paramexp'$) and
exploitation Q-values $\qin(s, z, a; \paramin)$ parametrized by $\paramin$ (and target network parameters $\paramin'$).
We train on trials with one exploration and one exploitation episode, but can test on arbitrarily many exploitation episodes, as the exploitation policy acts on each episode independently (i.e. it does not maintain a hidden state across episodes).
Using the choices for $F_\paramF$ and $q_\paramdecoder$ in \refsec{practical_dream}, with target parameters $\paramF'$ and $\paramdecoder'$ respectively, training proceeds as follows.

We first sample a new problem for the trial and roll-out the exploration policy, adding the roll-out to a replay buffer (lines 7-9).
Then, we roll-out the exploitation policy, adding the roll-out to a separate replay buffer (lines 10-12).
We train the exploitation policy on both stochastic encodings of the problem ID $\mathcal{N}(f_\paramF(\mdpindex), \rho^2 I)$ and on encodings of the exploration trajectory $g_\paramdecoder(\tauexp)$.

Next, we sample from the replay buffers and update the parameters.
First, we sample $(s_t, a_t, r_t, s_{t + 1}, \mdpindex, \tauexp)$-tuples from the exploration replay buffer and perform a normal DDQN update on the exploration Q-value parameters $\paramexp$ using rewards computed from the decoder (lines 13-15).
Concretely, we minimize the following standard DDQN loss function w.r.t., the parameters $\paramexp$, where the rewards are computed according to Equation \ref{eq:reward-exp}:
\begin{align*}
  &\sL_\text{exp}(\paramexp) = \E\left[\norm{\qexp(s_t, \tauexp_{:t}, a_t; \paramexp) - \text{target}}^2_2\right], \\
  \text{where } &\text{target} = (r_t^\text{exp} + \gamma \qexp(s_{t + 1}, \tauexp_{:t + 1}, a_\text{DDQN}; \paramexp'), \\
  &r_t^\text{exp} = \norm{f_\paramF(\mdpindex) - g_\paramdecoder(\tauexp_{:t})}^2_2 - \\
  &\quad\quad\;\;\, \norm{f_\paramF(\mdpindex) - g_\paramdecoder(\tauexp_{:t+1})}^2_2 - \rewardpenalty, \\
  &a_\text{DDQN} = \arg\max_a \qexp(s_{t + 1}, \tauexp_{:t + 1}, a; \paramexp).
\end{align*}
We perform a similar update with the exploitation Q-value parameters (lines 16--18).
We sample $(s, a, r, s', \mdpindex, \tauexp)$-tuples from the exploitation replay buffer and perform DDQN updates from the encoding of the problem ID and from the encoding of the exploration trajectory by minimizing the following losses:
\begin{align*}
  &\sL_\text{task-id}(\paramin, \paramF) = \E\left[\norm{\qin(s, f_\paramF(\mdpindex), a; \paramin) - \text{target}_\text{id}}^2_2\right], \\
  \text{where } &\text{target}_\text{id} = (r + \qin(s', f_{\paramF'}(\mdpindex), a_\text{id}; \paramin'), \\
  &a_\text{id} = \arg\max_a \qin(s', f_\paramF(\mdpindex), a; \paramin).
  \\
\end{align*}
\begin{align*}
  &\sL_\text{task-traj}(\paramin, \paramdecoder) = \\
  &\quad\quad\quad\E\left[\norm{\qin(s, g_\paramdecoder(\tauexp), a; \paramin) - \text{target}_\text{traj}}^2_2\right], \\
  \text{where } &\text{target}_\text{traj} = (r + \qin(s', g_{\paramdecoder'}(\tauexp), a_\text{traj}; \paramin'), \\
  &a_\text{traj} = \arg\max_a \qin(s', g_\paramdecoder(\tauexp), a; \paramin), \\
\end{align*}
Finally, from the same exploitation replay buffer samples, we also update the
problem ID embedder to enforce the information bottleneck (line 19)
and the decoder to approximate the true conditional distribution (line 20) by minimizing the following losses respectively:
\begin{align*}
    &\sL_\text{bottleneck}(\paramF) = \E_\mdpindex \left[\min{(\norm{f_\paramF(\mdpindex)}^2_2, K)}\right] \\
    \text{and } &\sL_\text{decoder}(\paramdecoder) =
    \E_{\tauexp}\left[\sum_t \norm{f_\paramF(\mdpindex) - g_\paramdecoder(\tauexp_{:t})}^2_2\right].
\end{align*}
Since the magnitude $\norm{f_\paramF(\mdpindex)}^2_2$ partially determines the scale of the reward, we add a hyperparameter $K$ and only minimize the magnitude when it is larger than $K$.
Altogether, we minimize the following loss:
\iclrupdate{\begin{align*}
    \sL(\paramexp, \paramin, \paramdecoder, \paramF) = &\sL_\text{exp}(\paramexp)\;+ \\
    &\sL_\text{task-traj}(\paramin, \paramdecoder) + \sL_\text{task-id}(\paramin, \paramF)\;+ \\
    &\sL_\text{bottleneck}(\paramF) + \sL_\text{decoder}(\paramdecoder).
\end{align*}
} 

As is standard with deep Q-learning~\citep{mnih2015human}, instead of sampling from the replay buffers
and updating after each episode, we sample and perform all of these updates every 4 timesteps.
We periodically update the target networks (lines 20-22).

\begin{algorithm*}
    \small
    \begin{flushleft}
    \begin{algorithmic}[1]
        \State \textbf{Initialize} exploitation replay buffer $\sB_{\text{task}} = \{\}$ and exploration replay buffer $\sB_{\text{exp}} = \{\}$
        \State \textbf{Initialize} exploitation Q-value $\qin$ parameters $\paramin$ and target network parameters $\paramin'$
        \State \textbf{Initialize} exploration Q-value $\qexp$ parameters $\paramexp$ and target network parameters $\paramexp'$
        \State \textbf{Initialize} problem ID embedder $f_\paramF$ parameters $\paramF$ and target parameters $\paramF'$
        \State \textbf{Initialize} trajectory embedder $g_\paramdecoder$ parameters $\paramdecoder$ and target parameters $\paramdecoder'$
        \For{$\textrm{trial} = 1$ \textbf{to} $\text{max trials}$}
          \State Sample problem $\mdpindex \sim p(\mdpindex)$, defining MDP $\langle\mathcal{S}, \mathcal{A}, \sR_\mdpindex, T_\mdpindex\rangle$

\State Roll-out $\epsilon$-greedy exploration policy $\qexp(s_t, \tauexp_{:t}, a_t; \paramexp)$, producing trajectory $\tauexp = (s_0, a_0, \ldots, s_T)$.
          \State Add tuples to the exploration replay buffer $\sB_{\text{exp}} = \sB_{\text{exp}} \cup \{(s_t, a_t, r_t, s_{t + 1}, \mdpindex, \tauexp)\}_t$.
          \item[]

          \State Compute embedding $z \sim \mathcal{N}(f_\paramF(\mdpindex), \rho^2 I)$ on $\text{trial} \equiv 0 \pmod{2}$ and $z = g_\paramdecoder(\tauexp)$ on $\text{trial} \equiv 1 \pmod{2}$.
          \State Roll-out $\epsilon$-greedy exploitation policy $\qin(s_t, z, a_t; \paramin)$, producing trajectory $(s_0, a_0, r_0, \ldots)$ with $r_t = \sR_\mdpindex(s_{t+1})$.
          \State Add tuples to the exploitation replay buffer $\sB_{\text{task}} = \sB_{\text{task}} \cup \{(s_t, a_t, r_t, s_{t + 1}, \mdpindex, \tauexp)\}_t$.
          \item[]

\State Sample batches of $(s_t, a_t, r_t, s_{t + 1}, \mdpindex, \tauexp) \sim \sB_{\text{exp}}$ from exploration replay buffer.
          \State Compute reward $r_t^\text{exp} = \norm{f_\paramF(\mdpindex) - g_\paramdecoder(\tauexp_{:t})}^2_2 - \norm{f_\paramF(\mdpindex) - g_\paramdecoder(\tauexp_{:t + 1})}^2_2 - \rewardpenalty$ (Equation~\ref{eq:reward-exp}).
          \State Optimize $\paramexp$ with DDQN update with tuple $(s_t, a_t, r_t^\text{exp}, s_{t + 1})$ with $\sL_\text{exp}(\paramexp)$
\item[]

          \State Sample batches of $(s, a, r, s', \mdpindex, \tauexp) \sim \sB_{\text{task-id}}$ from exploitation replay buffer.

          \State Optimize $\paramin$ and $\paramF$ with DDQN update with tuple
          $((s, \mdpindex), a, r, (s', \mdpindex))$ with $\sL_\text{task-id}(\paramin, \paramF)$ on $\text{trial} \equiv 0 \pmod{2}$.
          \State Optimize $\paramin$ and $\paramdecoder$ with DDQN update with tuple
          $((s, \tauexp), a, r, (s', \tauexp))$ with $\sL_\text{task-traj}(\paramin, \paramdecoder)$ on $\text{trial} \equiv 1 \pmod{2}$.

          \State Optimize $\paramF$ on $\sL_\text{bottleneck}(\paramF) = \nabla_\paramF \min(\norm{f_\paramF(\mdpindex)}^2_2, K)$
          \State Optimize $\paramdecoder$ on $\sL_\text{decoder}(\paramdecoder) = \nabla_\paramdecoder \sum_t{\norm{f_\paramF(\mdpindex) - g_\paramdecoder(\tauexp_{:t})}^2_2}$ (Equation~\ref{eqn:obj-exp})
          \item[]

          \If{$\text{trial} \equiv 0 \pmod{\text{target freq}}$}
            \State Update target parameters $\paramexp' = \paramexp$, $\paramin' = \paramin$, $\paramF' = \paramF$, $\paramdecoder' = \paramdecoder$
          \EndIf
        \EndFor
    \end{algorithmic}
    \end{flushleft}
    \caption{\ours DDQN}
    \label{alg:ours}
\end{algorithm*}

\section{Experiment Details}\label{sec:experiment_details}
\subsection{Problem Details}\label{sec:task_details}
\begin{figure*}\center
\includegraphics[width=\textwidth]{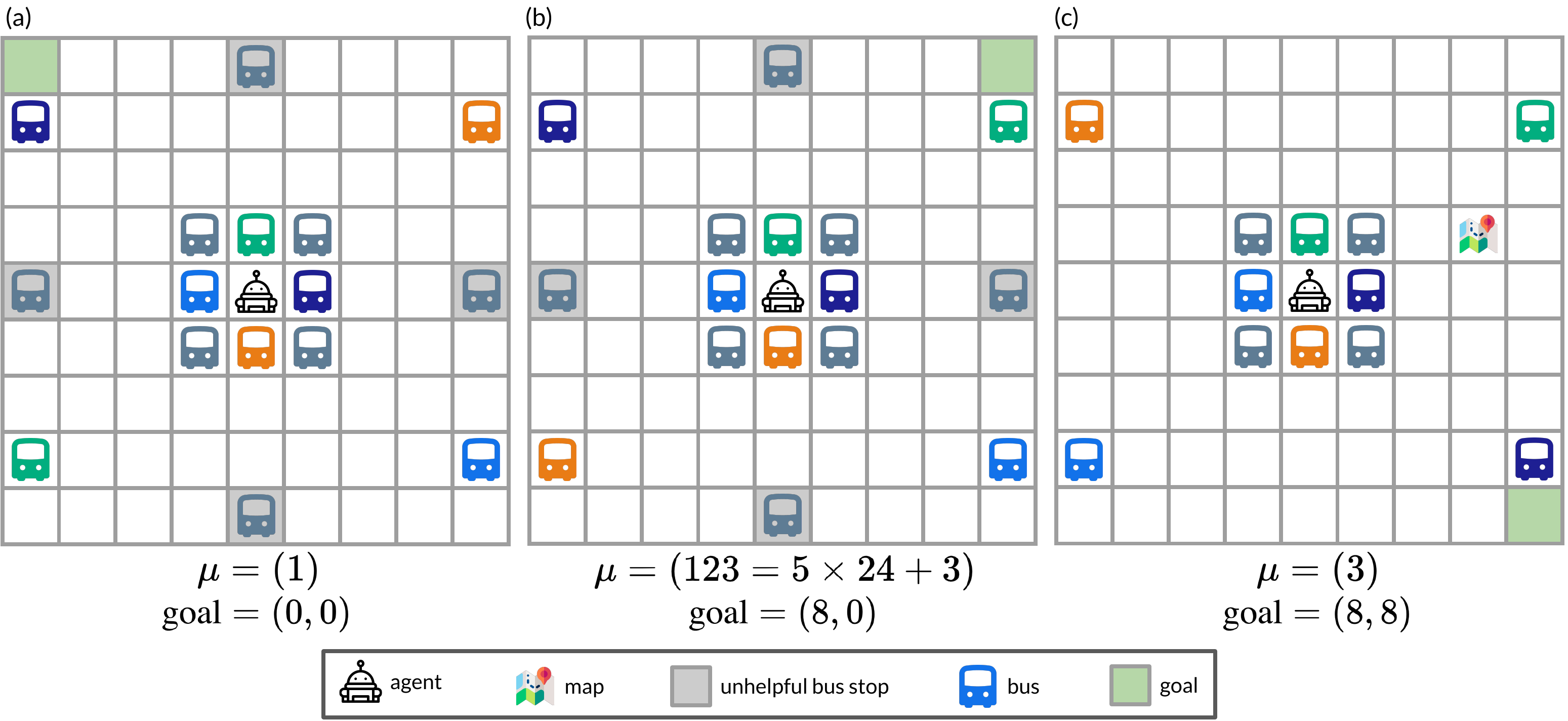}
\caption{Examples of different \emph{distracting bus} and \emph{map} problems.
(a) An example distracting bus problem.
Though all unhelpful distracting buses are drawn in the same color (gray),
the destinations of the gray buses are fixed within a problem.
(b) Another example distracting bus problem.
The destinations of the helpful colored buses are a different permutation
(the orange and green buses have swapped locations).
This takes on permutation $3 \equiv \mu \pmod{4!}$, instead of $1$.
The unhelpful gray buses are also a different permutation (not shown), taking on permutation $5 = \left \lfloor \frac{\mu}{4!} \right \rfloor$.
(c) An example map problem.
Touching the map reveals the destinations of the colored buses, by adding $\mu$ to the state observation.
}\label{fig:city_examples}
\end{figure*}

\paragraph{Distracting bus / map.}

Riding each of the four colored buses teleports the agent to near one of the green goal locations in the corners.
In different problems, the destinations of the colored buses change, but the bus positions and their destinations are fixed within each problem.
Additionally, in the distracting bus domain, the problem ID also encodes the destinations of the gray buses, which are permutations of the four gray locations on the midpoints of the sides.
More precisely, the problem ID $\mu \in \{0, 1, \ldots, 4! \times 4! = 576\}$ encodes both the permutation of the colored helpful bus destinations, indexed as $\mu \pmod{4!}$ and the permutation of the gray unhelpful bus destinations as $\left \lfloor \frac{\mu}{4!} \right \rfloor$.
We hold out most of the problem IDs during meta-training ($\frac{23}{24} \times 576 = 552$ are held-out for meta-training).

In the map domain, the problem $\mdpindex$ is an integer representing which of the 4! permutations of the four green goal locations the colored buses map to.
The states include an extra dimension, which is set to 0 when the agent is not at the map, and is set to this integer value $\mdpindex$ when the agent is at the map.
\reffig{city_examples} displays three such examples.

\paragraph{Cooking.}

\begin{figure}\center
  \begin{center}
    \includegraphics[width=\columnwidth]{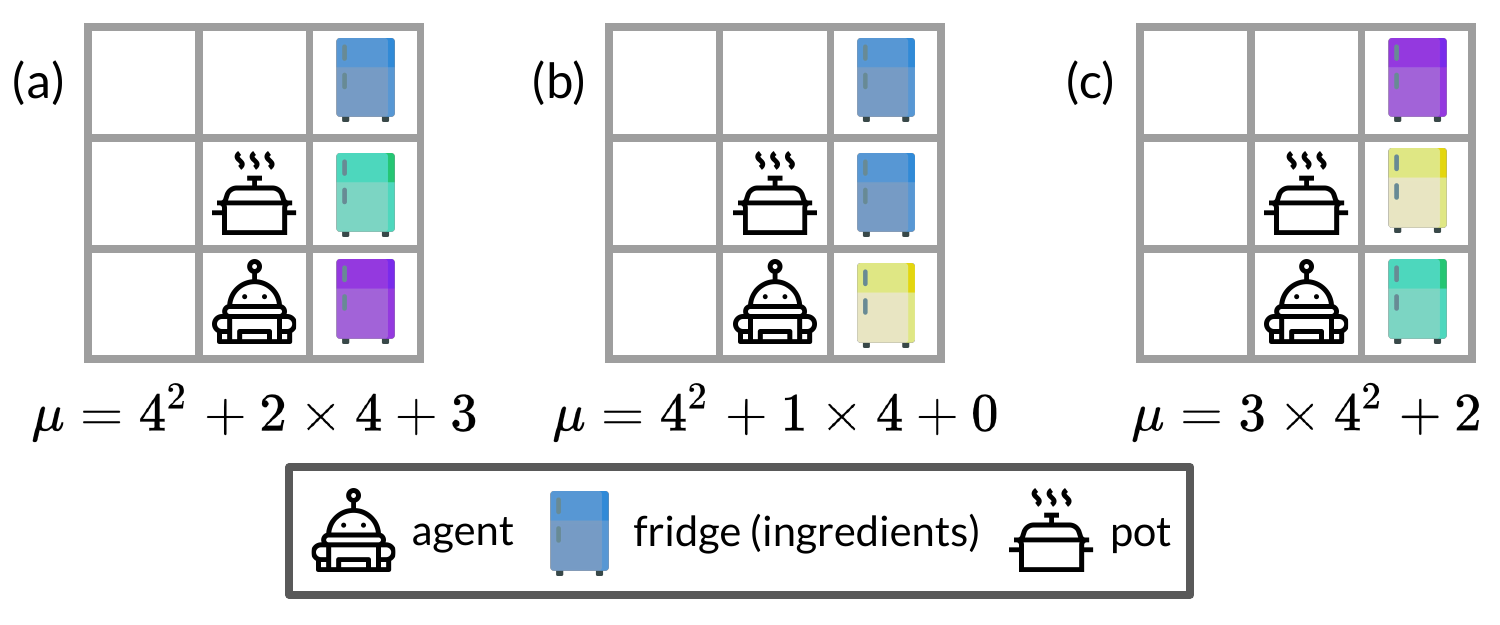}
  \end{center}
  \caption{
    Three example cooking problems.
    The contents of the fridges (color-coded) are different in different problems.
}
  \label{fig:cooking_examples}
\end{figure}

In different problems, the (color-coded) fridges contain 1 of 4 different ingredients.
The ingredients in each fridge are unknown until the goes to the fridge and uses the pickup action.
\reffig{cooking_examples} displays three example problems.
The problem ID $\mdpindex$ is an integer between $0$ and $4^3$, where $\mdpindex = 4^2a + 4b + c$ indicates that the top right fridge has ingredient $a$, the middle fridge has ingredient $b$ and the bottom right fridge has ingredient $c$.

The goals correspond to a recipe of placing the two correct ingredients in the pot in the right order.
Goals are tuples $(a, b)$, which indicate placing ingredient $a$ in the pot first, followed by ingredient $b$.
In a given problem, we only sample goals involving the recipes actually present in that problem.
During meta-training, we hold out a randomly selected problem $\mdpindex = 11$.

We use the following reward function $\sR_\mdpindex$.
The agent receives a per timestep penalty of $-0.1$ reward and receives $+0.25$ reward for completing each of the four steps:
(i) picking up the first ingredient specified by the goal;
(ii) placing the first ingredient in the pot;
(iii) picking up the second ingredient specified by the goal;
and (iv) placing the second ingredient in the pot.
To prevent the agent from gaming the reward function, e.g., by repeatedly picking up the first ingredient, dropping the first ingredient anywhere but in the pot yields a penalty of $-0.25$ reward, and similarly for all steps.
To encourage efficient exploration, the agent also receives a penalty of $-0.25$ reward for picking up the wrong ingredient.

\paragraph{Cooking without goals.}
While we evaluate on goal-conditioned benchmarks to deepen the exploration challenge, forcing the agent to discover all the relevant information for \emph{any} potential goal, many standard benchmarks~\citep{finn2017modelagnostic, yu2019meta} don't involve goals.
We therefore include a variant of the cooking task, where there are no goals.
We simply concatenate the goal (recipe) to the problem ID $\mdpindex$.
Additionally, we modify the rewards so that picking up the second ingredient yields $+0.25$ and dropping it yields $-0.25$ reward, so that it is possible to infer the recipe from the rewards.
Finally, to make the problem harder, the agent cannot pick up new ingredients unless its inventory is empty (by using the drop action), and we also increase the number of ingredients to $7$.
The results are in \refsec{additional_results}.

\paragraph{Sparse-reward 3D visual navigation.}

We implement this domain in Gym MiniWorld~\citep{boisvert2018gym}, where the agent's observations are $80 \times 60 \times 3$ RGB arrays.
There are three problems $\mdpindex = 0$ (the sign says ``blue''), $\mdpindex = 1$ (the sign says ``red''), and $\mdpindex = 2$ (the sign says ``green'').
There are two goals, represented as $0$ and $1$, corresponding to picking up the key and the box, respectively.
The reward function $\sA_\mdpindex(s, i)$ is +1 for picking up the correct colored object (according to $\mdpindex$) and the correct type of object (according to the goal) and $-1$ for picking up an object of the incorrect color or type.
Otherwise, the reward is 0.
On each episode, the agent begins at a random location on the other side of the barrier from the sign.

\subsection{Additional Results}\label{sec:additional_results}
\paragraph{Analysis of the learned policies.}
Please see \url{https://ezliu.github.io/dream/} for videos and analysis of the exploration and exploitation behavior learned by \ours and other approaches,
which is described in  the text below.

\paragraph{Distracting bus / map.}

\begin{figure*}\center
\includegraphics[width=\textwidth]{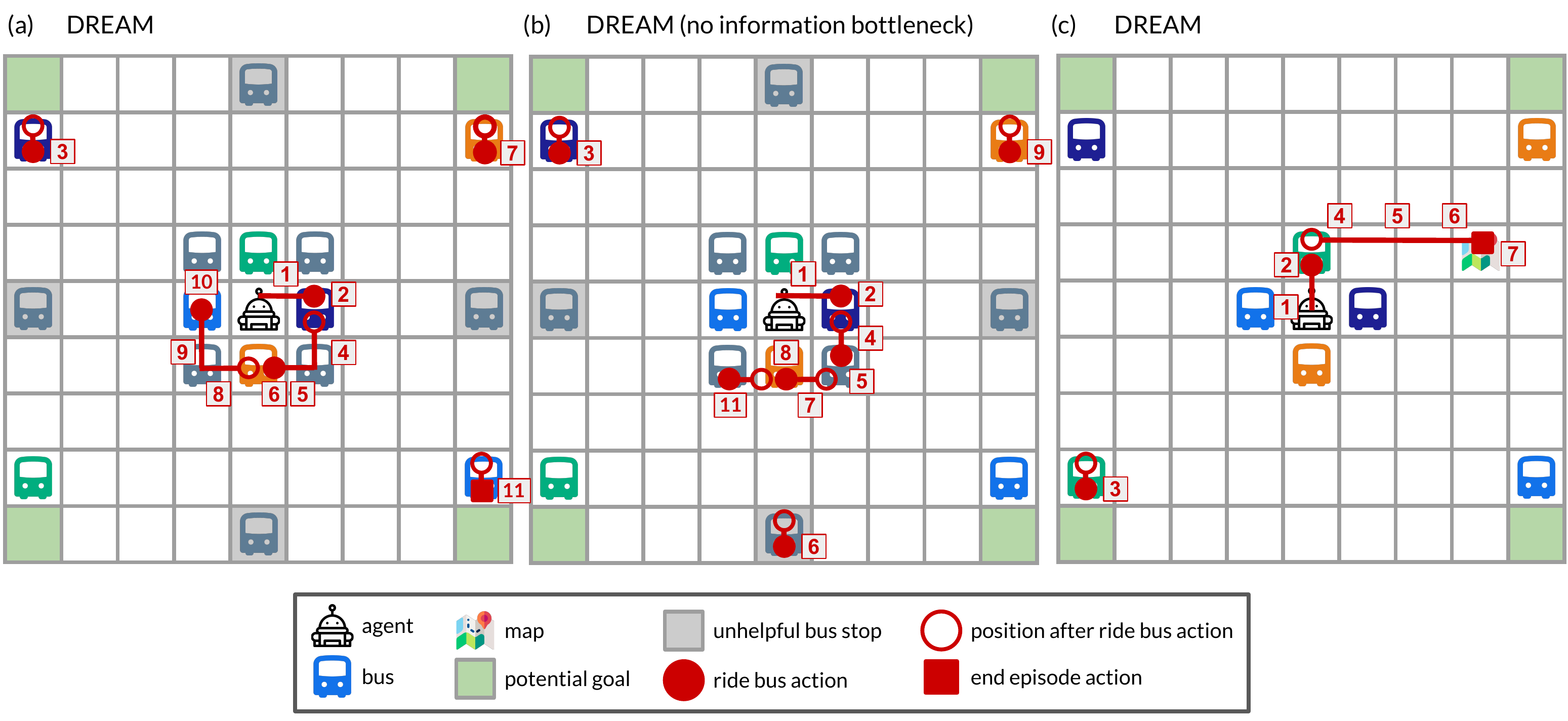}
\caption{
Examples of \ours's learned exploration behavior.
(a) \ours learns the optimal exploration behavior on the \emph{distraction} variant:
riding 3 of the 4 helpful colored buses, which allows it to infer the destinations of all colored buses and efficiently reach any goal during exploitation episodes.
(a) Without the information bottleneck, \ours also explores the unhelpful gray buses, since they are part of the problem.
This wastes exploration steps, and leads to lower returns during exploitation episodes.
(c) \ours learns the optimal exploration on the \emph{map} variant:
it goes to read the map revealing all the buses' destinations, and then ends the episode, though it unnecessarily rides one of the buses.
}\label{fig:city_qualitative_results}
\end{figure*}

\reffig{city_qualitative_results} shows the exploration policy \ours learns on the distracting bus and map domains.
With the information bottleneck, \ours optimally explores by riding 3 of the 4 colored buses and inferring the destination of the last colored bus (\reffig{city_examples}a).
Without the information bottleneck, \ours explores the unhelpful gray buses and runs out of time to explore all of the colored buses, leading to lower reward (\reffig{city_examples}b).
In the map domain, \ours optimally explores by visiting the map and terminating the exploration episode.
In contrast, the other methods (\rl, \import, \varibad) rarely visit the colored buses or map during exploration and consequently walk to their destination during exploitation, which requires more timesteps and therefore receives lower returns.

In \reffig{tsne}, we additionally visualize the exploration trajectory encodings $g_\paramdecoder(\tauexp)$ and problem ID encodings $f_\paramF(\mdpindex)$ that \ours learns in the distracting bus domain by applying t-SNE~\citep{maaten2008visualizing}.
We visualize the encodings of all possible problem IDs as dots.
They naturally cluster into 4! = 24 clusters, where the problems within each cluster differ only in the destinations of the gray distracting buses, and not the colored buses.
Problems in the support of the true posterior $p(\mdpindex \mid \tauexp)$ are drawn in green, while problems outside the support
(e.g., a problem that specifies that riding the green bus goes to location $(0, 1)$ when it has already been observed in $\tauexp$ that riding the orange bus goes to location $(0, 1)$)
are drawn in red.
We also plot the encoding of the exploration trajectory $\tauexp$ so far as a blue cross and the mean of the green clusters as a black square.
We find that the encoding of the exploration trajectory $g_\paramdecoder(\tauexp)$ tracks the mean of the green clusters until the end of the exploration episode, when only one cluster remains, and the destinations of all the colored buses has been discovered.
Intuitively, this captures uncertainty in what the potential problem ID may be.
More precisely, when the decoder is a Gaussian, placing $g_\paramdecoder(\tauexp)$ at the center of the encodings of problems in the support exactly minimizes Equation~\ref{eqn:obj-exp}.

\begin{figure*}\center
\includegraphics[width=\textwidth]{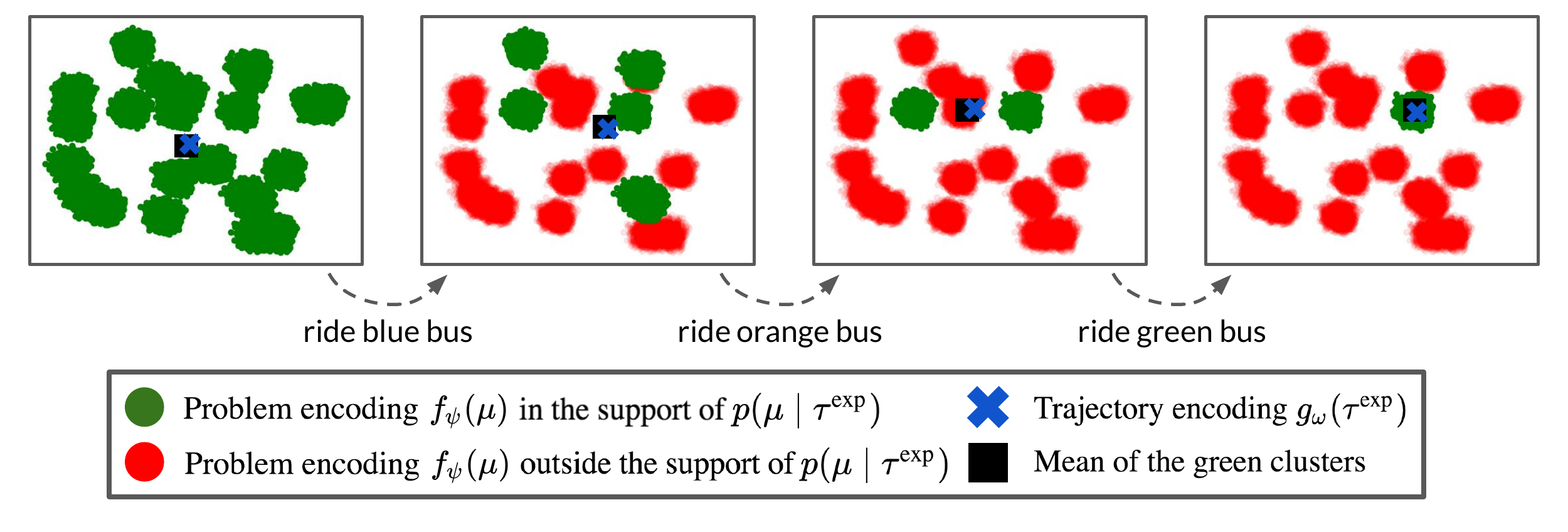}
\caption{\ours's learned encodings of the exploration trajectory and problems visualized with t-SNE~\citep{maaten2008visualizing}.
}\label{fig:tsne}
\end{figure*}

\begin{figure}\center
  \includegraphics[width=0.85\columnwidth]{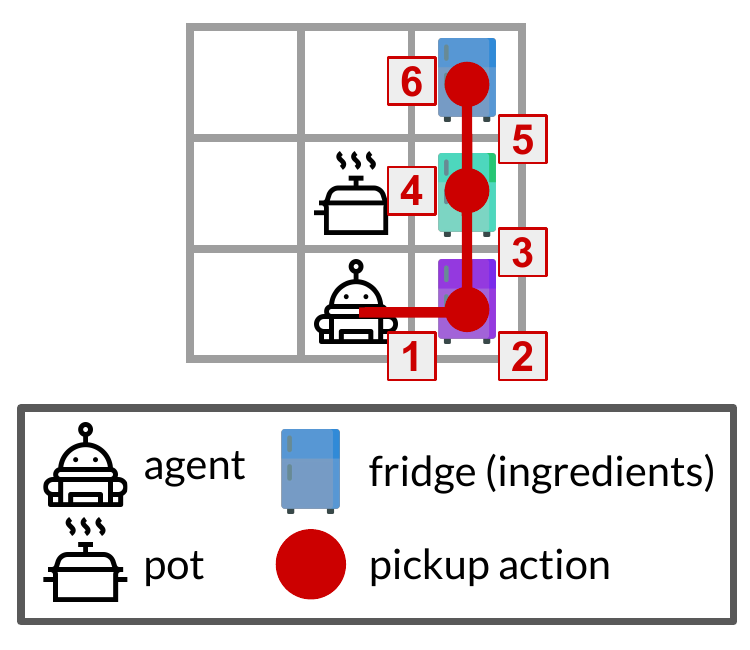}
  \captionof{figure}{
    \ours learns the optimal exploration policy, which learns the fridges' contents by going to each fridge and using the pickup action.
  }
  \label{fig:cooking_qualitative_results}
\end{figure}

\begin{figure}[t]
  \centering
  \includegraphics[width=0.85\columnwidth]{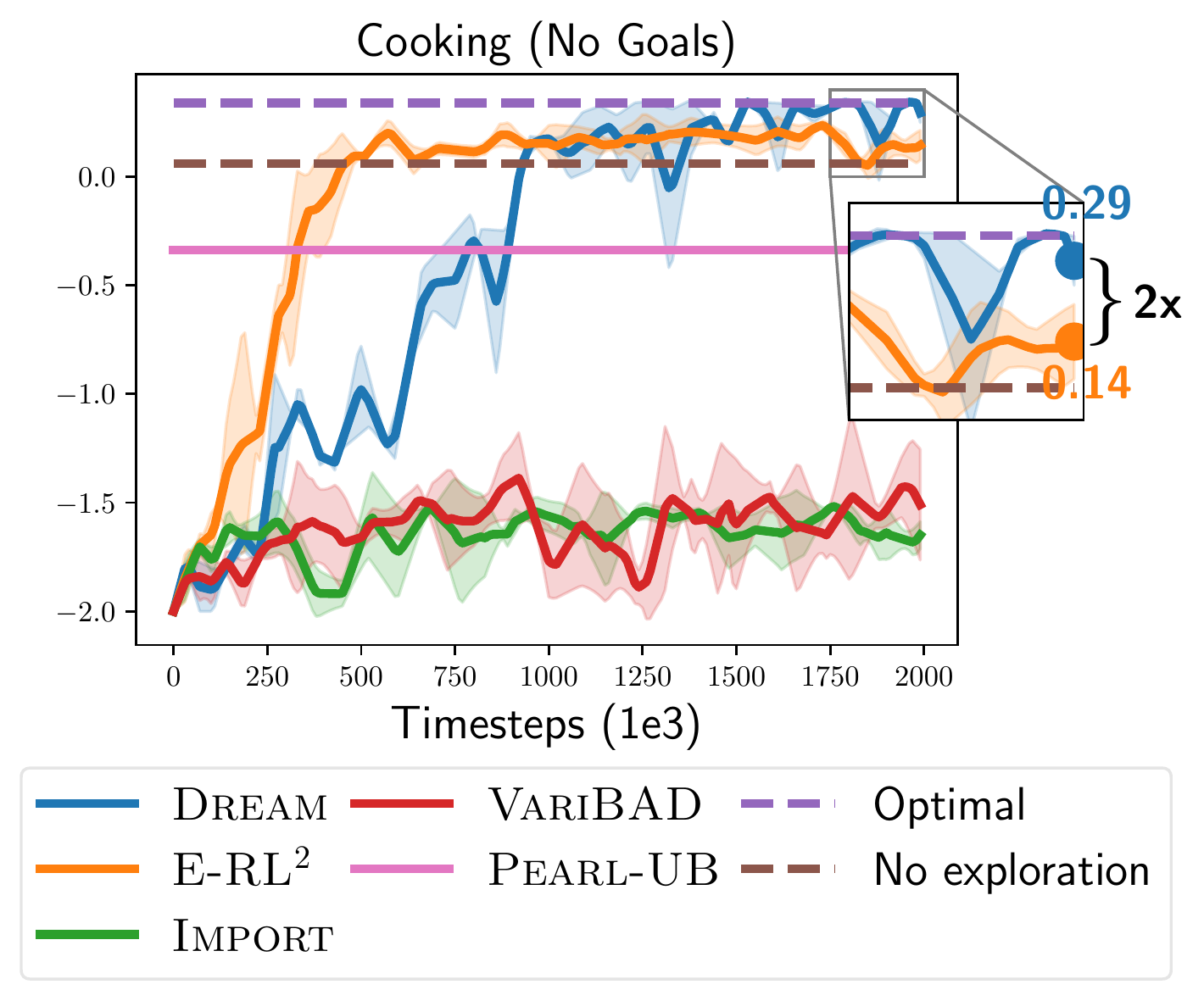}
  \captionof{figure}{Cooking without goals results. Only \ours learns the optimal policy, achieving \textasciitilde 2x more reward than the next best approach.}
  \label{fig:cooking_without_goals}
\end{figure}

\paragraph{Cooking.}

\reffig{cooking_qualitative_results} shows the exploration policy \ours learns on the cooking domain, which visits each of the fridges and investigates the contents with the "pickup" action.
In contrast, the other methods rarely visit the fridges during exploration, and instead determine the locations of the ingredients during exploitation, which requires more timesteps and therefore receives lower returns.

\paragraph{Cooking without goals.}
We provide additional results in the case where the cooking domain is modified to not include goals
(see \refsec{task_details}).
The results are summarized in \reffig{cooking_without_goals} and show the same trends as the results in original cooking benchmark.
\ours learns to optimally explore by investigating the fridges, and then also optimally exploits, by directly collecting the relevant ingredients.
The next best approach E-\rl, only sometimes explores the fridges, again getting stuck in a local optimum, yielding only slightly higher reward than no exploration at all.

\paragraph{Sparse-reward 3D visual navigation.}

\ours optimally explores by walking around the barrier and reading the sign.
See \url{https://ezliu.github.io/dream/} for videos.
The other methods do not read the sign at all and therefore cannot solve the problem.

\paragraph{Robustness to imperfections in the problem ID.}
Recall that the problem ID is a simple and easy-to-provide unique one-hot for each problem.
We test of \ours is robust to imperfections in the problem ID by assigning each problem in the \emph{map} benchmark $3$ different problem IDs.
When a problem is sampled during meta-training, it is randomly labeled with 1 of these $3$ different problem IDs.
We find that this imperfection in the problem ID does not impact \ours's final performance at all:
it still achieves optimal exploitation returns of $0.8$ after $1$M time steps of training.

\cameraupdate{\paragraph{Robustness to hyperparameters.} \ours uses 3 additional hyperparameters on top of standard RL algorithm hyperparameters:
the information bottleneck weight $\lambda$, encoder and decoder variance $\rho$, and per timestep exploration penalty $c$.
We test how sensitive \ours is to these hyperparameter values by evaluating \ours with 3 different values for each of these hyperparameters, while holding the other hyperparameter values constant, set to  the values described in \refsec{hyperparams}.
Specifically, we evaluate the following values:
$\lambda = [0.1, 1, 3], \rho = [0.01, 0.1, 0.3], c = [0, 0.01, 0.1]$.
Across all four domains (\emph{distracting bus}, \emph{map},  \emph{cooking}, and \emph{3D visual navigation}), \ours achieves near-optimal returns for all values except $\lambda = 3$, suggesting that \ours is fairly robust to hyperparameter choices.
}

\subsection{Other Approaches and Architecture Details}\label{sec:baseline_details}
In this section, we detail the loss functions that E-\rl, \import, and \varibad optimize, as well as the model architectures used in our experiments.
Where possible, we use the same model architecture for all methods:
\ours, E-\rl, \import, and \varibad.
All approaches are implemented in PyTorch~\citep{paszke2017automatic}, using a DQN implementation adapted from \citet{liu2020learning} and code adapted from \citet{liu2020imitation}.

\paragraph{State and problem ID embeddings.}
All approaches use the same method to embed the state and problem ID.
For these embeddings, we embed each dimension independently with an embedding matrix of output dimension 32.
Then, we concatenate the per-dimension embeddings and apply two linear layers with output dimensions 256 and 64 respectively, with ReLU activations.

In the 3D visual navigation task, we use a different embedding scheme for the states, as they are images.
We apply 3 CNN layers, each with 32 output layers and stride length 2, and with kernel sizes of 5, 5, and 4 respectively.
We apply ReLU activations between the CNN layers and apply a final linear layer to the flattened output of the CNN layers, with an output dimension of 128.

All state and problem ID embeddings below use this scheme.

\paragraph{Experience embeddings.}
E-\rl, \import, \varibad and the exploration and exploitation policies in \ours also learn an embedding of the history of prior experiences $\tau=(s_0, a_0, r_0, s_1, \ldots)$ and current state $s_T$.
To do this, we first separately embed each $(s_{t + 1}, a_t, r_t, d_t)$-tuple, where $d_t$ is an episode termination flag (true if the episode ends on this experience, and false otherwise), as follows:
\begin{itemize}
    \item Embed the state $s_t$ as $e(s_t)$, using the state embedding scheme described above.
    \item Embed the action $a_t$ as $e(a_t)$ with an embedding matrix of output dimension 16.
    We set $a_{-1}$ to be 0.
    \item Embed the rewards with a linear layer of output dimension 16.
We set $r_{-1}$ to be 0.
    \item Embed the episode termination $d_t$ as $e(d_t)$ with an embedding matrix of output dimension 16.
    Note that $d$ is true during all episode terminations within a trial for \rl, \import, and \varibad.
\end{itemize}
Then, we apply a final linear layer with output dimension 64 to the concatenation of the above
$[e(s_t); e(a_t); e(r_t); d_t]$.
Finally, to obtain an embedding of the entire history $\tau$, we embed each experience separately as above, and then pass an LSTM with hidden dimension 64 over the experience embeddings, where the initial hidden and cell states are set to be 0-vectors.

\paragraph{\ours.}
For the decoder $g_\paramdecoder(\tauexp=(s_0, a_0, r_0, s_1, \ldots, s_T))$, we embed each transition $(s_t, a_t, r_t, s_{t + 1})$ of the exploration trajectory $\tauexp$ using the same embedding scheme as above, except we also embed the next state $s_{t + 1}$.
Then, given embeddings for each transition, we embed the entire trajectory by passing an LSTM with output dimension 128 on top of the transition embeddings, followed by two linear layers of output dimension 128 and 64 with ReLU activations.

For the exploitation policy Q-values $\qin_\paramin(a_t \mid s_t, \tau_{:t}, z)$,
during meta-testing, we choose $z$ to be the decoder embedding of the exploration trajectory $g_\paramdecoder(\tauexp)$, and during meta-training, we choose $z$ to be the the embedding of the problem ID $e_\paramin(\mdpindex)$.
To embed the history $\tau_{:t} = (s_0, a_0, r_0, \ldots, s_{t - 1}, a_{t - 1}, r_{t - 1})$, we use the experience embedding scheme described above to obtain an embedding $e(\tau_{:t})$, omitting the reward embeddings for simplicity.
We embed the state with a learned embedding functions $e(s)$.
Then we apply a linear layer of output dimension 64 to the concatenation of $[e(s_t); e(\tau_{:t}), z]$ with a ReLU activation.
Finally, we apply two linear layer heads of output dimension 1 and $|\sA|$ respectively to form estimates of the value and advantage functions, using the dueling Q-network parametrization.
To obtain Q-values, we add the value function to the advantage function, subtracting the mean of the advantages.

For the exploration policy Q-values $\qexp_\paramexp(a_t \mid s_t, \tauexp_{:t})$,
we embed the $s_t$ and $\tauexp_{:t}$ according to the embedding scheme above.
Then, we apply two linear layer heads to obtain value and advantage estimates as above.

\paragraph{E-\rl.}
E-\rl learns a policy $\pi(a_t \mid s_t, \tau_{:t})$ producing actions $a_t$ given the state $s_t$ and history $\tau_{:t}$.
Like with all approaches, we parametrize this with dueling double Q-networks, learning Q-values
$\hat{Q}(s_t, \tau_{:t}, a_t)$.
We embed the current state $s_t$ and history $\tau_{:t}$ using the embedding scheme described above (with episode termination embeddings).
Then, we apply two final linear layer heads to obtain value and advantage estimates.

\paragraph{\import}
\import also learns a recurrent policy $\pi(a_t \mid s_t, z)$, but conditions on the embedding $z$, which is either an embedding of the problem $\mdpindex$ or the history $\tau_{:t}$.
We also parametrize this policy with dueling double Q-networks, learning Q-values $\hat{Q}(s_t, z, a_t)$.
We embed the state $s_t$ as $e(s_t)$, the problem $\mdpindex$ as $e_\phi(\mdpindex)$ and the history $\tau_{:t}$ as $e_\theta(\tau_{:t})$ using the previously described embedding schemes.
Then we alternate meta-training trials between choosing $z = e_\phi(\mdpindex)$ and $z = e_\theta(\tau_{:t})$.
We apply a linear layer of output dimension 64 to the concatenation $[e(s_t); z]$ with ReLU activations and then apply two linear layer heads to obtain value and advantage estimates.

Additionally, \import uses the following auxiliary loss function to encourage the history embedding $e_\theta(\tau_{:t})$ to be close to the problem embedding $e_\phi(\mdpindex)$ (optimized only w.r.t., $\theta$):
\begin{align*}
    \sL_{\import}(\theta) = \beta \E_{(\tau, \mdpindex)}\left[\sum_t \norm{e_\theta(\tau_{:t}) - e_\phi(\mdpindex)}^2_2\right],
\end{align*}
where $\tau$ is a trajectory from rolling out the policy on problem $\mdpindex$.
Following \citet{kamienny2020learning}, we use $\beta = 1$ in our final experiments, and found that performance changed very little when we experimented with other values of $\beta$.

\paragraph{\varibad.}
\varibad also learns a recurrent policy $\pi(a_t \mid z)$, but over a \emph{belief state} z capturing the history $\tau_{:t}$ and current state $s_t$.
We also parametrize this dueling double Q-networks, learning Q-values $\hat{Q}(s_t, z, a_t)$.

\varibad encodes the belief state with an encoder $\text{enc}(z \mid s_t, \tau{:t})$.
Our implementation of this encoder embeds $s_t$ and $\tau_{:t}$ using the same experience embedding approach as above, and use the output as the mean $m$ for a distribution.
Then, we set $\text{enc}(z \mid s_t, \tau{:t}) = \sN(m, \nu^2I)$, where $\nu^2 = 0.00001$.
We also tried learning the variance instead of fixing it to $\nu^2I$ by applying a linear head to the output of the experience embeddings, but found no change in performance, so stuck with the simpler fixed variance approach.
Finally, after sampling $z$ from the encoder, we also embed the current state $s_t$ as $e(s_t)$ and apply a linear layer of output dimension 64 to the concatenation $[e(s_t); z]$.
Then, we apply two linear layer heads to obtain value and advantage estimates.

\varibad does not update the encoder via gradients through the policy.
Instead, \varibad jointly trains the encoder with state decoder $\hat{T}(s' \mid a, s, z)$ and reward decoder $\hat{\sR}(s' \mid a, s, z)$, where $z$ is sampled from the encoder, as follows.
Both decoders embed the action $a$ as $e(a)$ with an embedding matrix of output dimension 32 and embed the state $s$ as $e(s)$.
Then we apply two linear layers with output dimension 128 to the concatenation $[e(s); e(a); z]$.
Finally, we apply two linear heads, one for the state decoder and one for the reward decoder and take the mean-squared error with the true next state $s'$ and the true rewards $r$ respectively.
In the 3D visual navigation domain, we remove the state decoder, because the state is too high-dimensional to predict.
Note that \citet{zintgraf2019varibad} found better results when removing the state decoder in all experiments.
We also tried to remove the state decoder in the grid world experiments, but found better performance when keeping the state decoder.
We also found that \varibad performed better without the KL loss term, so we excluded that for our final experiments.

\subsection{Hyperparameters}\label{sec:hyperparams}
\begin{table}[]
    \centering
    \begin{tabular}{cc}
        Hyperparameter & Value \\
        \toprule
        Discount Factor $\gamma$ & 0.99 \\
        Test-time $\epsilon$ & 0 \\
        Learning Rate & 0.0001 \\
        Replay buffer batch size & 32 \\
        Target parameters syncing frequency & 5000 updates \\
        Update frequency & 4 steps \\
        Grad norm clipping & 10
    \end{tabular}
    \caption{Hyperparameters shared across all methods: \ours, \rl, \import, and \varibad.}
    \label{tab:hyperparameters}
\end{table}

In this section, we detail the hyperparameters used in our experiments.
Where possible, we used the default DQN hyperparameter values from \citet{mnih2015human}.
and shared the same hyperparameter values across all methods for fairness.
We optimize all methods with the Adam optimizer~\citep{kingma2015adam}.
\reftab{hyperparameters} summarizes the shared hyperparameters used by all methods and we detail any differences in hyperparameters between the methods below.

All methods use a linear decaying $\epsilon$ schedule for $\epsilon$-greedy exploration.
For \rl, \import, and \varibad, we decay $\epsilon$ from $1$ to $0.01$ over $500000$ steps.
For \ours, we split the decaying between the exploration and exploitation policies.
We decay each policy's $\epsilon$ from $1$ to 0.01 over 250000 steps.

We train the recurrent policies (\ours's exploration and exploitation policies, \rl, \import, and \varibad) with a simplified version of the methods in \citet{kapturowski2019recurrent} by storing a replay buffer with up to 16000 sequences of 50 consecutive timesteps.
We decrease the maximum size from 16000 to 10000 for the 3D visual navigation experiments in order to fit inside a single NVIDIA GeForce RTX 2080 GPU.

For \ours, we additionally use per timestep exploration reward penalty $c = 0.01$,
decoder and stochastic encoder variance $\rho^2 = 0.1$,
and information bottleneck weight $\lambda = 1$.
\iclrupdate{Note that this information bottleneck weight $\lambda$ could be adapted via dual gradient descent to solve the constrained optimization problem in \refeqn{bottleneck-problem}, but we find that dynamically adjusting $\lambda$ is not necessary for good performance.
} For the MiniWorld experiments, we use $c = 0$.

\section{Analysis}\label{sec:proofs}
\subsection{Consistency}\label{sec:consistency_proof}

We restate the consistency result of \ours (\refsec{consistency}) and prove it below.

\consistency*

\begin{proof}
Recall that $\piin_\star$ and $F_\star(z \mid \mdpindex)$ are optimized to solve the constrained optimization in \refeqn{bottleneck-problem}.
In particular, they must satisfy the constraint, so  $\piin_\star$ achieves the desired expected returns conditioned on the stochastic encoding of the problem $F_\star(z \mid \mdpindex)$:
\begin{equation*}
  \E_{\iclrupdate{\mdpindex \sim p(\mdpindex),} z \sim F_\star(z \mid \mdpindex)}\left[V^{\piin_\star}(z; \mdpindex)\right]
  = \E_{\mdpindex \sim p(\mdpindex)}\left[V^{*}(\mdpindex) \right],
\end{equation*}
where $V^{\piin_\star}(z; \mdpindex)$ is the expected returns of $\piin_\star$ on problem $\mdpindex$ given embedding $z$.
Therefore, it suffices to show that the distribution over $z$ from the decoder $q_\star(z \mid \tauexp)$ is equal to the distribution from the encoder $F_\star(z \mid \mdpindex)$ for all exploration trajectories in the support of $\piexp(\tauexp \mid \mdpindex)$\footnote{We slightly abuse notation to use $\piexp(\tauexp \mid \mdpindex)$ to denote the distribution of exploration trajectories $\tauexp$ from rolling out $\piexp$ on problem $\mdpindex$.}, for each problem $\mdpindex$.
Then,
\begin{align*}
  &\E_{\mdpindex \sim p(\mdpindex), \tauexp \sim \piexp_\star, z \sim q_\star(z \mid \tauexp)}\left[V^{\piin_\star}(z; \mdpindex) \right] \\
  &\quad = \E_{\iclrupdate{\mdpindex \sim p(\mdpindex),} z \sim F_\star(z \mid \mdpindex)}\left[V^{\piin_\star}(z; \mdpindex)\right] \\
  &\quad = \E_{\mdpindex \sim p(\mdpindex)}\left[V^{*}(\mdpindex) \right]
\end{align*}
as desired.
We show that this occurs as follows.

Given stochastic encoder $F_\star(z \mid \mdpindex)$, exploration policy $\piexp_\star$ maximizes $I(\tauexp; z) = H(z) - H(z \mid \tauexp)$ (Equation~\ref{eqn:obj-exp}) by assumption.
Since only $H(z \mid \tauexp)$ depends on $\piexp_\star$, the exploration policy outputs trajectories that minimize
\begin{align*}
  &H(z \mid \tauexp) \\
  &\quad = \E_{\mdpindex \sim p(\mdpindex)}\left[
    \E_{\tauexp \sim \piexp(\tauexp \sim \mdpindex)}\left[
      \E_{z \sim F_\star(z \mid \mdpindex)}\left[-\log p(z \mid \tauexp)\right]
    \right]
  \right] \\
  &\quad = \E_{\mdpindex \sim p(\mdpindex)}\left[
    \E_{\tauexp \sim \piexp(\tauexp \sim \mdpindex)}\left[
      H(F_\star(z \mid \mdpindex), p(z \mid \tauexp))
    \right]
  \right],
\end{align*}
where $p(z \mid \tauexp)$ is the true conditional distribution and $H(F_\star(z \mid \mdpindex), p(z \mid \tauexp))$ is the cross-entropy.
The cross-entropy is minimized when $p(z \mid \tauexp) = F_\star(z \mid \mdpindex)$, which can be achieved with long enough exploration trajectories $T$ if $\langle \mathcal{S}, \mathcal{A}, \mathcal{R}_\mdpindex, \mathcal{T}_\mdpindex \rangle$ is ergodic (by visiting each transition sufficiently many times).
Optimized over an expressive enough function class, $q_\star(z \mid \tauexp)$ equals the true conditional distribution $p(z \mid \tauexp)$ at the optimum of Equation~\ref{eqn:obj-exp}, which equals $F_\star(z \mid \mdpindex)$ as desired.
\end{proof}

\subsection{Tabular Example}\label{sec:tabular_example}

We first formally detail a more general form of the simple tabular example in \refsec{sample_complexity}, where episodes are horizon $H$ rather than 1-step bandit problems.
Then we prove sample complexity bounds for \rl and \ours, with $\epsilon$-greedy tabular Q-learning with $\epsilon = 1$, i.e., uniform random exploration.

\paragraph{Setting.}
We construct this horizon $H$ setting so that taking a \emph{sequence} of $H$ actions $\actions_\star$ (instead of a single action as before) in the exploration episode leads to a trajectory $\tauexp_\star$ that reveals the problem $\mdpindex$;
all other action sequences $\actions$ lead to a trajectory $\tauexp_\actions$ that reveals no information.
Similarly, the problem $\mdpindex$ identifies a unique sequence of $H$ actions
$\actions_\mdpindex$ that receives reward $1$ during exploitation, while all other action sequences receive reward $0$.
Again, taking the action sequence $\actions_\star$ during exploration is therefore necessary and sufficient to obtain optimal reward $1$ during exploitation.

We formally construct this setting by considering a family of episodic MDPs $\langle \sS, \sA, \sR_\mdpindex, T_\mdpindex \rangle$ parametrized by the problem ID $\mdpindex \in \problems$, where:
\begin{itemize}
  \item Each exploitation and exploration episode is horizon $H$.
  \item The action space $\sA$ consists of $A$ discrete actions $\{1, 2, \ldots, A\}$.
  \item The space of problems $\problems = \{1, 2, \ldots, |\sA|^H\}$ and the distribution $p(\mdpindex)$ is uniform.
\end{itemize}

To reveal the problem via the optimal action sequence $\actions_\star$ and to allow $\actions_\mdpindex$ to uniquely receive reward, we construct the state space and deterministic dynamics as follows.
\begin{itemize}
  \item States $s \in \sS$ are $(H + 2)$-dimensional and the deterministic dynamics are constructed so the first index represents the current timestep $t$, the middle $H$ indices represent the history of actions taken, and the last index reveals the problem ID if $\actions_\star$ is taken.
    The initial state is the zero vector $s_0 = \bzero$ and we denote the state at the $t$-th timestep $s_t$ as $(t, a_0, a_1, \ldots, a_{t - 1}, 0, \ldots, 0, 0)$.
  \item The optimal exploration action sequence $\actions_\star$ is set to be taking action $|\sA|$ for $H$ timesteps.
    In problem $\mdpindex$ taking action $a_{H - 1} = 1$ at state
    $s_{H - 1} = (H - 1, a_0 = 1, \ldots, a_{H - 2} = 1, 0, 0)$
    (i.e., taking the entire action sequence $\actions_\star$) transitions to the state
    $s_H = (H, a_0 = 1, \ldots, a_{H - 2} = 1, a_{H - 1} = 1, \mdpindex)$,
    revealing the problem $\mdpindex$.
  \item The action sequence $\actions_\mdpindex$ identified by the problem $\mdpindex$ is set as the problem $\mdpindex$ in base $|\sA|$:
    i.e., $\actions_\mdpindex$ is a sequence of $H$ actions $(a_0, a_1, \ldots, a_{H - 1})$ with $\sum_{t = 0}^{H - 1} a_t |\sA|^t = \mdpindex$.
    In problem $\mdpindex$ with $\actions_\mdpindex = (a_0, a_1, \ldots, a_{H - 1})$,
    taking action $a_{H - 1}$ at timestep $H - 1$ at state
    $s_{H - 1} = (H - 1, a_0, a_1, \ldots, a_{H - 2}, 0, 0)$
    (i.e., taking the entire action sequence $\actions_\mdpindex$) yields
    $\sR_\mdpindex(s_{H - 1}, a_{H - 1}) = 1$.
    Reward is $0$ everywhere else: i.e., $\sR_\mdpindex(s, a) = 0$ for all other states $s$ and actions $a$.
  \item With these dynamics, the exploration trajectory $\tauexp_\actions = (s_0, a_0, r_0, \ldots, s_H)$ is uniquely identified by the action sequence $\actions$ and the problem $\mdpindex$ if revealed in $s_H$.
    We therefore write $\tauexp_\actions = (\actions, \mdpindex)$ for when $\actions = \actions_\star$ reveals the problem $\mdpindex$, and $\tauexp_\actions = (\actions, 0)$, otherwise.
\end{itemize}

\paragraph{Uniform random exploration.}
In this general setting, we analyze the number of samples required to learn the optimal exploration policy with \rl and \ours via $\epsilon$-greedy tabular Q-learning.
We formally analyze the simpler case where $\epsilon = 1$, i.e., uniform random exploration, but empirically find that \ours only learns faster with smaller $\epsilon$, and \rl only learns slower.

In this particular tabular example with deterministic dynamics that encode the entire action history and rewards, learning a per timestep Q-value is equivalent to learning a Q-value for the entire trajectory.
Hence, we denote exploration Q-values $\qexp(\actions)$ estimating the returns from taking the entire sequence of $H$ actions $\actions$ at the initial state $s_0$ and
exeuction Q-values $\qin(\tauexp, \actions)$ estimating the returns from taking the entire sequence of $H$ actions $\actions$ at the initial state $s_0$ given the exploration trajectory $\tauexp$.
We drop $s_0$ from notation, since it is fixed.

Recall that \rl learns exploration Q-values $\qexp$ by regressing toward the exploitation Q-values $\qin$.
We estimate the exploitation Q-values $\qin(\tauexp, \actions)$ as the sample
mean of returns from taking actions $\actions$ given the exploration
trajectory $\tauexp$ and estimate the exploration Q-values $\qexp(\actions)$ as the sample mean of the targets.
More precisely, for action sequences $\actions \neq \actions_\star$, the
resulting exploration trajectory $\tauexp_\actions$ is deterministically $(\actions, 0)$, so we set
$\qexp(\actions) = \vin(\tauexp_\actions) = \max_{\actions'} \qin(\tauexp_\actions, \actions')$.
For $\actions_\star$, the resulting exploration trajectory $\tauexp_{\actions_\star}$ may be any of $(\actions_\star, \mdpindex)$ for
$\mdpindex \in \problems$, so we set $\qexp(\actions_\star)$ as the empirical mean of $\vin(\tauexp_{\actions_\star})$ of observed $\tauexp_{\actions_\star}$.

Recall that \ours learns exploration Q-values $\qexp$ by regressing toward the learned decoder $\log \hat{q}(\mdpindex \mid \tauexp_\actions)$.
We estimate the decoder $\hat{q}(\mdpindex \mid \tauexp_\actions)$ as the empirical counts of $(\mdpindex, \tauexp_\actions)$ divided by the empirical counts of $\tauexp_\actions$ and similarly estimate the Q-values as the empirical mean of $\log \hat{q}(\mdpindex \mid \tauexp_\actions)$.
We denote the exploration Q-values learned after $t$ timesteps as $\qexp_t$, and similarly denote the estimated decoder after $t$ timesteps as $\hat{q}_t$.

Given this setup, we are ready to state the formal sample complexity results.
Intuitively, learning the exploitation Q-values for \rl is slow, because, in problem $\mdpindex$, it involves observing the optimal exploration trajectory from taking actions $\actions_\star$ and then observing the corresponding exploitation actions $\actions_\mdpindex$, which only jointly happens roughly once per $|\sA|^{2H}$ samples.
Since \rl regresses the exploration Q-values toward the exploitation Q-values, the exploration Q-values are also slow to learn.
In contrast, learning the decoder $\hat{q}(\mdpindex \mid \tauexp_\actions)$ is much faster, as it is independent of the exploitation actions, and in particular, already learns the correct value from a single sample of $\actions_\star$.
We formalize this intuition in the following proposition, which shows that \ours learns in a factor of at least $|\sA|^H |\problems|$ fewer samples than \rl.

\begin{prop}
Let $T$ be the number of samples from uniform random exploration such that the greedy-exploration policy is guaranteed to be optimal (i.e., $\arg\max_{\actions} \qexp_t(\actions) = \actions_\star$) for all $t \geq T$.
If $\qexp$ is learned with \ours, the expected value of $T$ is $\mathcal{O}(|\sA|^H \log{|\sA|^H})$.
If $\qexp$ is learned with \rl, the expected value of $T$ is $\Omega(|\sA|^{2H} |\problems| \log{|\sA|^{H}})$.
\end{prop}

\begin{proof}
  For \ours, $\qexp_T(\actions_\star) > \qexp_T(\actions)$ for all $\actions \neq \actions_\star$ if
  $\log \hat{q}_T(\mdpindex \mid (\actions_\star, \mdpindex)) > \log \hat{q}_T(\mdpindex \mid (\actions, 0))$
  for all $\mdpindex$ and $\actions \neq \actions_\star$.
  For all $t \geq T$, $\qexp_t$ is guaranteed to be optimal, since no sequence of samples will cause
  $\log \hat{q}_t(\mdpindex \mid (\actions_\star, \mdpindex)) = 0 \leq \log \hat{q}_t(\mdpindex \mid (\actions, 0))$ for any $\actions \neq \actions_\star$.
  This occurs once we've observed $(\mdpindex, (\actions, 0))$ for two distinct $\mdpindex$ for each $\actions \neq \actions_\star$ and we've observed $(\mdpindex, (\actions_\star, \mdpindex))$ for at least one $\mdpindex$.
  We can compute an upperbound on the expected number of samples required to observe $(\mdpindex, \tauexp_\actions)$ for two distinct $\mdpindex$ for each action sequence $\actions$ by casting this as a coupon collector problem, where each pair $(\mdpindex, \tauexp_\actions)$ is a coupon.
  There are $2|\sA|^H$ total coupons to collect.
  This yields that the expected number of samples is $\sO(|\sA|^H \log{|\sA|^H})$.

  For \rl, the exploration policy is optimal for all timesteps $t \geq T$ for some $T$ only if the exploitation values $\vin_T(\tauexp=(\actions_\star, \mdpindex)) = 1$ for all $\mdpindex$ in $\problems$.
  Otherwise, there is a small, but non-zero probability that $\vin_t(\tauexp=(\actions, 0))$ will be greater at some $t > T$.
  For the exploitation values to be optimal at all optimal exploration trajectories $\vin_T(\tauexp=(\actions_\star, \mdpindex)) = 1$ for all $\mdpindex \in \problems$,
  we must jointly observe exploration trajectory $\tauexp = (\actions_\star, \mdpindex)$ and corresponding action sequence $\actions_\mdpindex$ for each problem $\mdpindex \in \problems$.
  We can lower bound the expected number of samples required to observe this by casting this as a coupon collector problem, where each pair $(\tauexp=(\actions_\star, \mdpindex), \actions_\mdpindex)$ is a coupon.
  There are $|\problems| \cdot |\sA|^H$ unique coupons to collect and collecting any coupon only occurs with probability $\frac{1}{|\sA|^H}$ in each episode.
  This yields that the expected number of samples is $\Omega(|\sA|^{2H} \cdot|\problems| \cdot\log{|\sA|^{H}})$.
\end{proof}

\end{document}